\documentclass[10pt,twocolumn,letterpaper]{article}

\usepackage{iccv}
\usepackage{times}
\usepackage{epsfig}
\usepackage{graphicx}
\usepackage{amsmath}
\usepackage{amssymb}
\usepackage{epsfig,amsbsy,amsthm}
\usepackage[export]{adjustbox}
\usepackage{subfig}
\usepackage{bbm}
\usepackage{stmaryrd}
\usepackage{url}

\newcommand{\papertitle}{Sublabel-Accurate Discretization of Nonconvex Free-Discontinuity Problems}

\newcommand{\bitem}{\begin{itemize}}
\newcommand{\eitem}{\end{itemize}}

\newcommand{\bpm}{\begin{pmatrix}}      
\newcommand{\epm}{\end{pmatrix}}

\newcommand{\tmop}[1]{\ensuremath{\operatorname{#1}}}
\newcommand{\norm}[1]{\Vert #1 \Vert}
\newcommand{\normc}[1]{| #1 |}
\newcommand{\bbR}{\mathbb{R}}
\newcommand{\bbRext}{\mathbb{R} \cup \{ \infty \} }

\providecommand{\iprod}[2]{\langle#1,#2\rangle}

\newcommand{\bi}{\begin{itemize}}
\newcommand{\ei}{\end{itemize}}

\newcommand{\Ss}{\mathcal{S}}

\newcommand{\cref}[1]{ {\tiny[{#1}]}}
\newcommand{\tm}[1]{}

\newcommand{\dom}{\mathsf{dom}}

\newcommand{\epi}{\mathsf{epi}}

\newcommand{\beq}{\begin{equation}}
\newcommand{\eeq}{\end{equation}}
\newcommand{\beqa}{\begin{eqnarray}}
\newcommand{\eeqa}{\end{eqnarray}}
\newcommand{\bc}{\begin{center}}
\newcommand{\ec}{\end{center}}

\newcommand{\BV}{\tmop{BV}}
\newcommand{\SBV}{\tmop{SBV}}

\newcommand \TV         {{TV}}                              % TV

\newcommand{\Div}{\tmop{Div}}

\newtheorem{prop}{Proposition}

\newcommand{\vl}{\boldsymbol{v}}

\newcommand{\one}[1]{\mathbf{1}_{#1}}

\newcommand{\measurerestr}{%
  \,\raisebox{-.127ex}{\reflectbox{\rotatebox[origin=br]{-90}{$\lnot$}}}\,%
}

\newcommand{\iver}[1]{\llbracket {#1} \rrbracket}
\graphicspath{{./figures/}}

\iccvfinalcopy % *** Uncomment this line for the final submission

\def\iccvPaperID{502} % *** Enter the ICCV Paper ID here

\ificcvfinal\fi
\begin{document}

\newcommand{\figConvexExact}{
\begin{figure}[t!]
  \centering
  \captionsetup[subfloat]{labelformat=empty,justification=centering,singlelinecheck=false}
  \subfloat[][\scriptsize Direct Optimization\\$E_{\text{Q}}=2002.9$]{
    \includegraphics[width=0.11\textwidth,trim={0 5cm 0 1cm},clip]{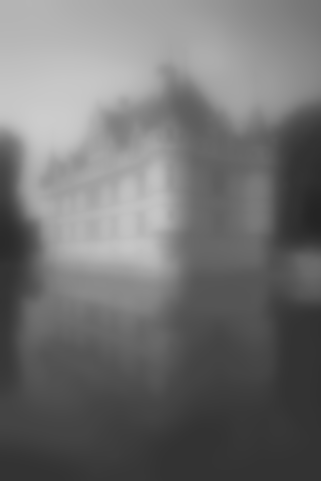}
  }
  \subfloat[][\scriptsize \cite{PCBC-SIIMS}, $\ell=2$\\ $E_{\text{Q}}=15708.3$]{
    \includegraphics[width=0.11\textwidth,trim={0 5cm 0 1cm},clip]{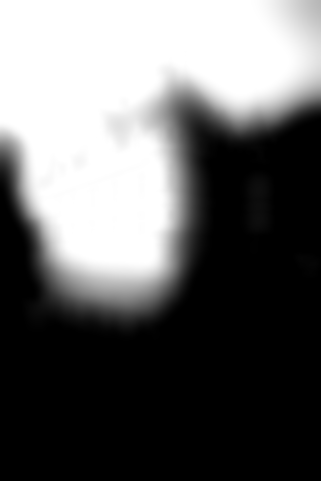}
  }
  \subfloat[][\scriptsize \cite{PCBC-SIIMS}, $\ell=3$\\ $E_{\text{Q}}=5103.8$]{
    \includegraphics[width=0.11\textwidth,trim={0 5cm 0 1cm},clip]{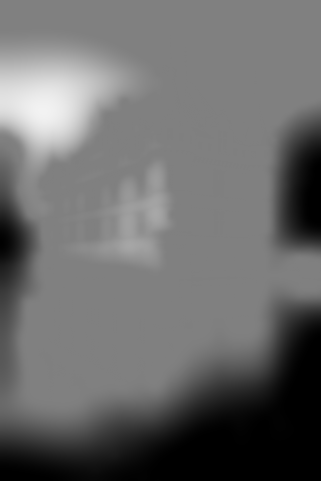}
  }
  \subfloat[][\scriptsize \cite{PCBC-SIIMS}, $\ell=5$\\ $E_{\text{Q}}=2415.9$]{
    \includegraphics[width=0.11\textwidth,trim={0 5cm 0 1cm},clip]{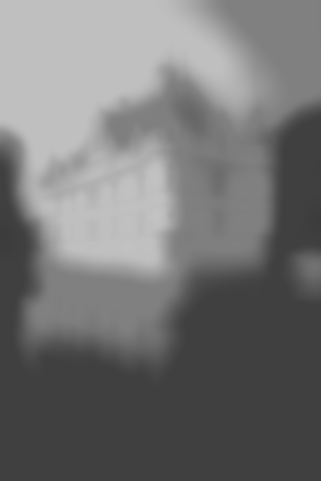}
  }\\
  \subfloat[][\scriptsize \cite{PCBC-SIIMS}, $\ell=16$\\ $E_{\text{Q}}=2016.5$]{
    \includegraphics[width=0.11\textwidth,trim={0 5cm 0 1cm},clip]{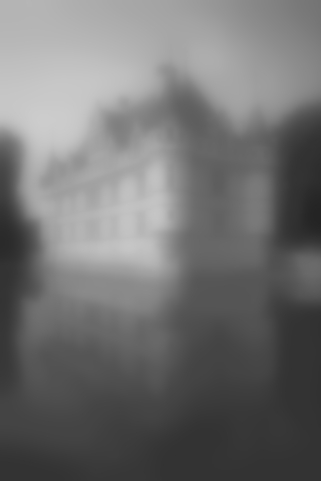}
  }
  \subfloat[][\scriptsize Proposed, $\ell=2$\\ $E_{\text{Q}}=2002.9$]{
    \includegraphics[width=0.11\textwidth,trim={0 5cm 0 1cm},clip]{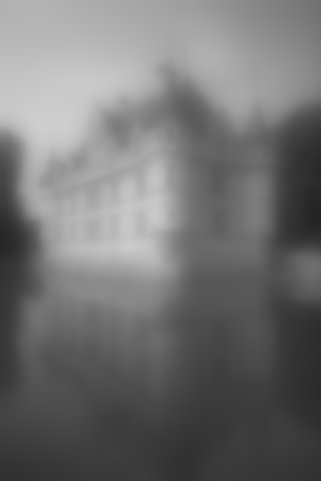}
  }
  \subfloat[][\scriptsize Proposed, $\ell=3$\\ $E_{\text{Q}}=2002.9$]{
    \includegraphics[width=0.11\textwidth,trim={0 5cm 0 1cm},clip]{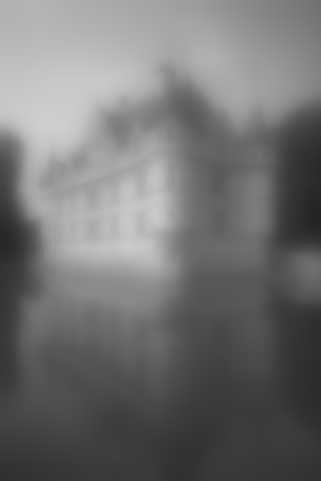}
  }
  \subfloat[][\scriptsize Proposed, $\ell=5$\\ $E_{\text{Q}}=2002.9$]{
    \includegraphics[width=0.11\textwidth,trim={0 5cm 0 1cm},clip]{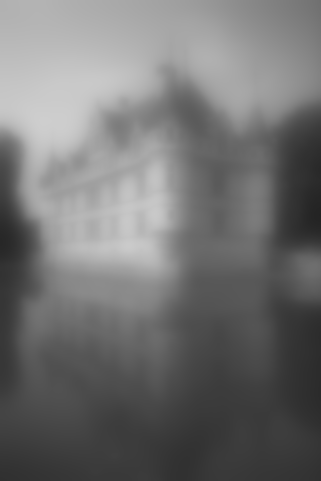}
  }%\\[-1mm]
  \caption{To verify the tightness of the approximation, we optimize a convex problem (quadratic data term with
    quadratic regularization). The discretization with piecewise linear $\varphi_t$ recovers the 
exact solution with $2$ labels and remains tight (numerically) for all $\ell > 2$, while the 
traditional discretization from \cite{PCBC-SIIMS} leads to a strong label bias.} 
  \label{fig:convex_exact}
\end{figure}
}

\newcommand{\figFEM}{
\begin{figure}[t!]
  \centering
  \captionsetup[subfloat]{justification=centering,singlelinecheck=false}
  \subfloat[]{
    \includegraphics[width=0.46\textwidth]{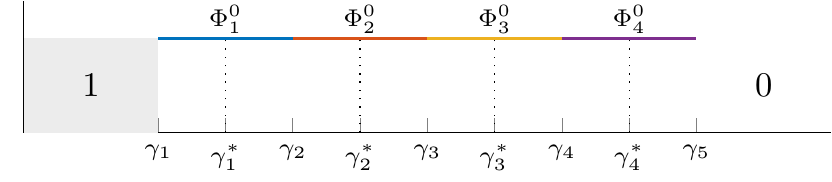}
  }\\[-0mm]
  \subfloat[]{
    \includegraphics[width=0.46\textwidth]{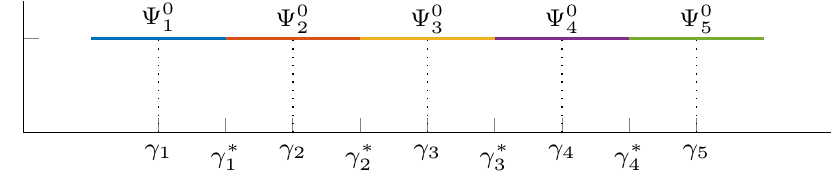}
  }\\[-0mm]
  \subfloat[]{
    \includegraphics[width=0.46\textwidth]{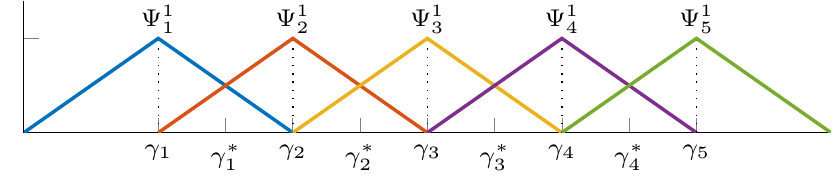}
  }%\\[-1mm]
  \caption{Overview of the notation and proposed finite dimensional approximation spaces. } 
  \label{fig:fem}
\end{figure}
}

\newcommand{\figConstantVsLinear}{
\begin{figure*}[t!]
  \centering
  \captionsetup[subfloat]{labelformat=empty,justification=centering,singlelinecheck=false}
  \subfloat[]{
    \includegraphics[width=0.485\textwidth]{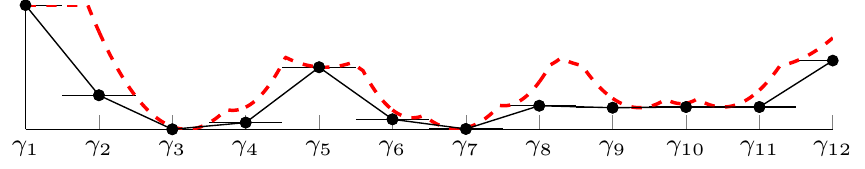}
  }
  \captionsetup[subfloat]{labelformat=empty,justification=centering,singlelinecheck=false}
  \subfloat[]{
    \includegraphics[width=0.485\textwidth]{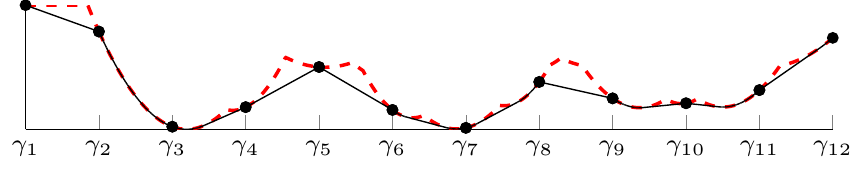}
  }%\\[-1mm]
  \caption{\textbf{Left:} piecewise constant dual variables $\varphi_t$ lead to a linear
    approximation (shown in black) to the original cost function (shown in red). The unaries
    are determined through min-pooling of the continuous cost in the Voronoi cells around the labels. \textbf{Right:} continuous piecewise linear dual variables $\varphi_t$ convexify the costs on each interval.} %In Sec.~\ref{sec:piecw_quad} we show how to exactly represent the convex envelope of piecewise quadratic functions (as the one in red) in the numerical implementation. }
  \label{fig:constant_vs_linear}
\end{figure*}
}

\newcommand{\figGraph}{
\begin{figure}[t!]
  \centering
  \captionsetup[subfloat]{labelformat=empty,justification=centering,singlelinecheck=false}
  \subfloat[]{
    \includegraphics[width=0.46\textwidth]{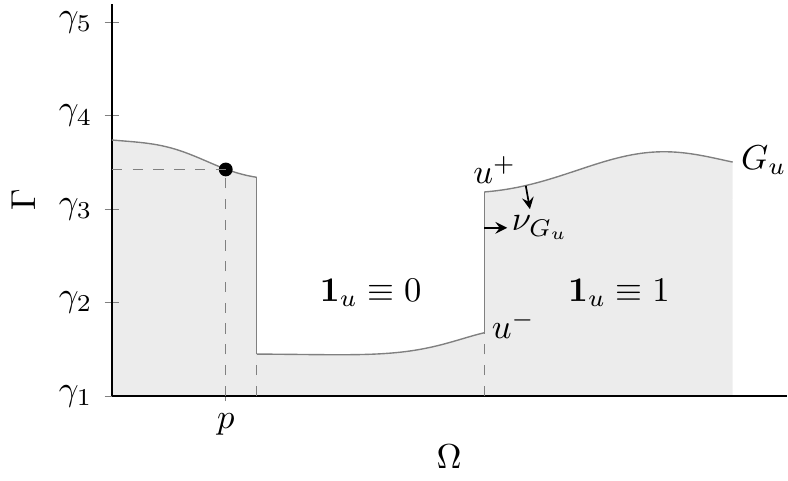}
  }%\\[-1mm]

  \caption{The central idea behind the convex relaxation for problem \eqref{eq:unlifted_cont_mshah_general}
    is to reformulate the functional in terms of the complete graph $G_u \subset \Omega \times \Gamma$ of $u : \Omega \to \Gamma$ in the product space. This procedure is often referred to as ``lifting'', as one lifts the dimensionality of the problem.
  } 
  \label{fig:graph}
\end{figure}
}

\newcommand{\figJointStereoSegm}{
\begin{figure*}[t!]
  \centering
  \subfloat[][\scriptsize Left input image]{
    \includegraphics[width=0.189\textwidth]{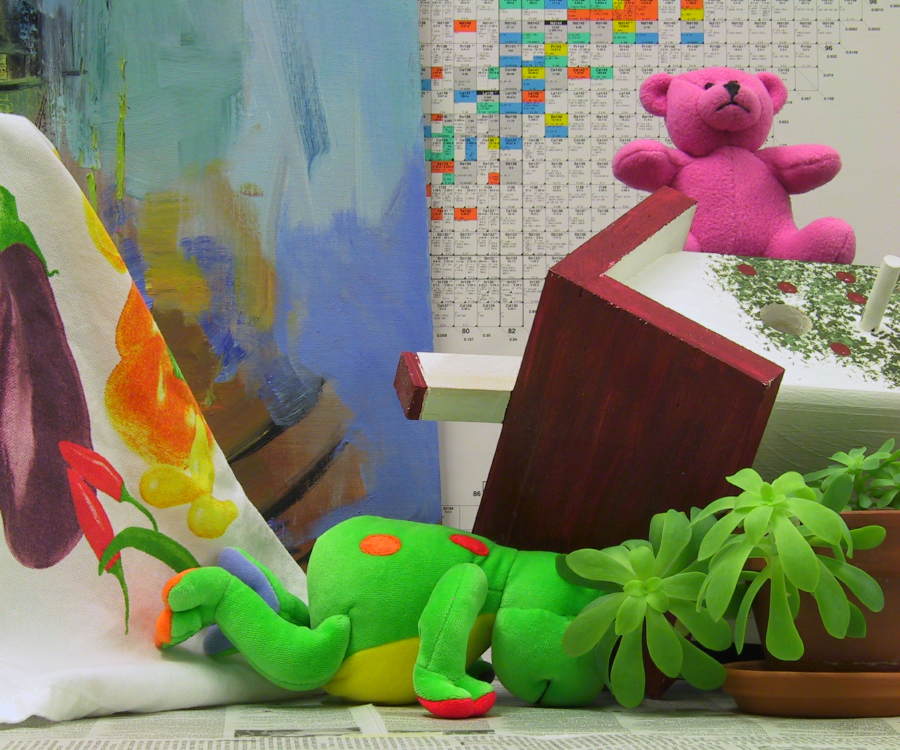}
  }
  \subfloat[][\scriptsize Proposed, (Segmentation)]{
    \includegraphics[width=0.189\textwidth]{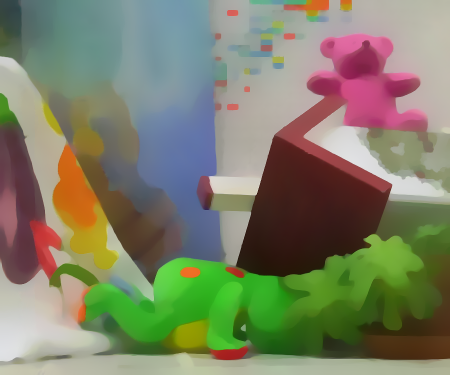}
  }
  \subfloat[][\scriptsize Proposed, (Depth map)]{
    \includegraphics[width=0.189\textwidth]{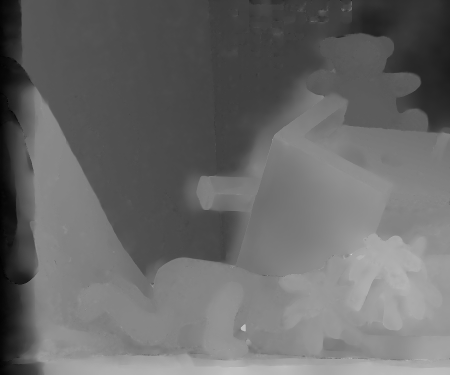}
  }
  \subfloat[][\scriptsize \cite{Strekalovskiy-et-al-cvpr12}, (Segmentation) ]{
    \includegraphics[width=0.189\textwidth]{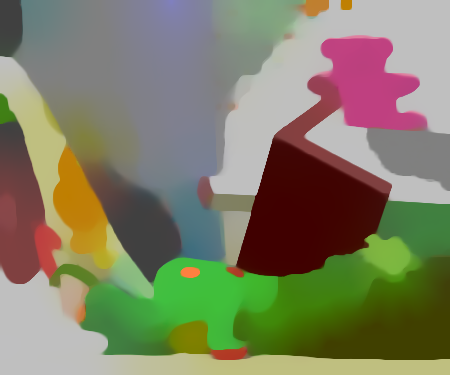}
  }
  \subfloat[][\scriptsize \cite{Strekalovskiy-et-al-cvpr12}, (Depth map)]{
    \includegraphics[width=0.189\textwidth]{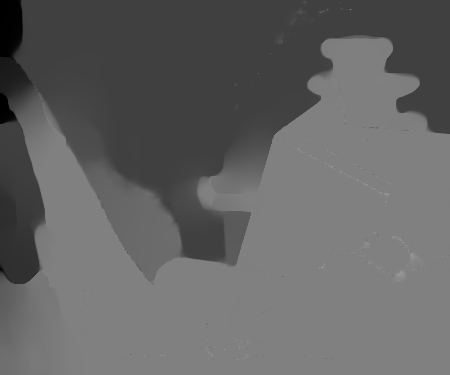}
  }%\\[-1mm]
  \caption{Joint segmentation and stereo matching.
    \textbf{b), c)} Using the proposed discretization we can arrive at
    smooth solutions using a moderate ($5 \times 5 \times 5 \times 5$)
    discretization of the
    $4$-dimensional RGB-D label space. \textbf{d), e)} When using such a coarse sampling of the label space, the classical discretization used in \cite{Strekalovskiy-et-al-cvpr12} leads to a strong
    label bias. Note that with the proposed approach, a piecewise constant segmentation as in \textbf{d)} could also be obtained by increasing the smoothness parameter. \vspace{-3mm}} 
  \label{fig:joint_stereo_segm}
\end{figure*}
}

\newcommand{\figDenoiseSynthetic}{
\begin{figure*}[t!]
  \centering
  \captionsetup[subfloat]{labelformat=empty,justification=centering,singlelinecheck=false}
  \subfloat[][\scriptsize Noisy Input, (PSNR=$10.4$)]{
    \includegraphics[width=0.132\textwidth]{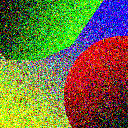}
  }
  \subfloat[][\scriptsize \cite{Strekalovskiy-et-al-cvpr12}, $\ell=2\times2\times2$ (PSNR=$14.7$)]{
    \includegraphics[width=0.133\textwidth]{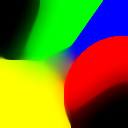}
  }
  \subfloat[][\scriptsize \cite{Strekalovskiy-et-al-cvpr12}, $\ell=4\times4\times4$ (PSNR=$25.0$)]{
    \includegraphics[width=0.133\textwidth]{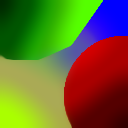}
  }
  \subfloat[][\scriptsize \cite{Strekalovskiy-et-al-cvpr12}, $\ell=6\times6\times6$ (PSNR=$29.3$)]{
    \includegraphics[width=0.133\textwidth]{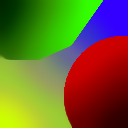}
  }
  \subfloat[][\scriptsize Ours, $\ell=2\times2\times2$, (PSNR=$24.8$)]{
    \includegraphics[width=0.133\textwidth]{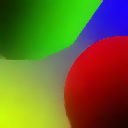}
  }
  \subfloat[][\scriptsize Ours, $\ell=4\times4\times4$, (PSNR=$28.0$)]{
    \includegraphics[width=0.133\textwidth]{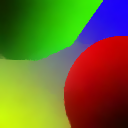}
  }
  \subfloat[][\scriptsize Ours, $\ell=6\times6\times6$, (PSNR=$\mathbf{30.0}$)]{
    \includegraphics[width=0.133\textwidth]{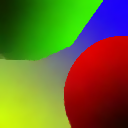}
  }%\\[-1mm]
 \caption{Denoising of a synthetic piecewise smooth image degraded
   with $30\%$ Gaussian noise. The standard discretization of the
   vectorial Mumford-Shah functional shows a strong bias towards the
   chosen labels (see also Figure~\ref{fig:denoise_syntheticii}), while the
   proposed discretization has no bias and leads to the highest
   overall peak signal to noise ratio (PSNR). \vspace{-1mm} }
  \label{fig:denoise_synthetic}
\end{figure*}
}

\newcommand{\figDenoiseSyntheticii}{
\begin{figure}[t!]
  \centering
  \captionsetup[subfloat]{labelformat=empty,justification=centering,singlelinecheck=false}
  \subfloat[][]{
    \includegraphics[width=0.23\textwidth]{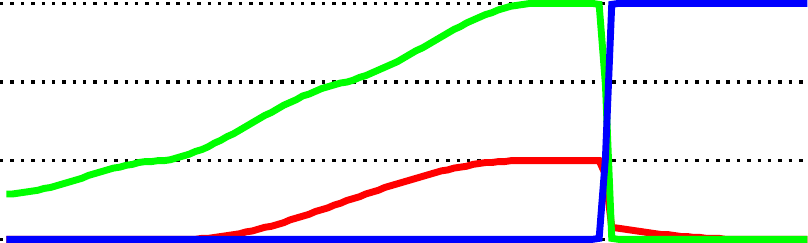}
  }
  \subfloat[][]{
    \includegraphics[width=0.23\textwidth]{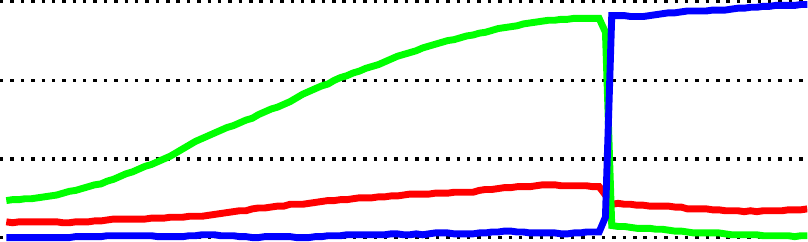}
  }%\\[-2mm]
  \caption{We show a 1D-slice through the resulting image in
    Figure~\ref{fig:denoise_synthetic} (with $\ell=4\times4\times4$). The discretization \cite{Strekalovskiy-et-al-cvpr12} (left)
  shows a strong bias towards the labels, while the proposed
  discretization (right) yields a sublabel-accurate solution.}% to the
%  vectorial Mumford-Shah functional.}
  \label{fig:denoise_syntheticii}
\end{figure}
}

%%%%%%%%% TITLE
\title{\papertitle}

\author{Thomas M\"ollenhoff \qquad Daniel Cremers\\
Technical University of Munich\\
{\tt\small \{thomas.moellenhoff,cremers\}@tum.de}
}

\maketitle

\begin{abstract}
In this work we show how sublabel-accurate multilabeling approaches
\cite{laude16eccv, moellenhoff-laude-cvpr-2016} can be derived by
approximating a classical label-continuous convex relaxation of
nonconvex free-discontinuity problems. This insight allows to extend
these sublabel-accurate approaches from total variation to general
convex and nonconvex regularizations. Furthermore, it leads to a
systematic approach to the discretization of continuous convex
relaxations. We study the relationship to existing discretizations and 
to discrete-continuous MRFs. Finally, we apply the proposed approach to obtain a 
sublabel-accurate and convex solution to the vectorial Mumford-Shah functional
and show in several experiments that it leads to more precise
solutions using fewer labels.
\end{abstract}

%%%%%%%%% BODY TEXT

\section{Introduction}
\subsection{A class of continuous optimization problems}
Many tasks particularly in low-level computer vision can be formulated as optimization problems over mappings $u : \Omega \to \Gamma$ between sets $\Omega$ and $\Gamma$.
The energy functional is usually designed in such a way that the minimizing argument corresponds to a mapping with the desired solution properties. 
In classical discrete Markov random field (MRF) approaches, which we refer to as \emph{fully discrete optimization}, $\Omega$ is typically a set of nodes (e.g., pixels or superpixels) and $\Gamma$ a set of
labels $\{1, \hdots, \ell\}$. 

However, in many problems such as image denoising, stereo matching or optical flow where $\Gamma \subset \bbR^d$ is naturally modeled as a continuum, this discretization into \emph{labels} can entail 
unreasonably high demands in memory when using a fine sampling, or it leads to a strong label bias when using a coarser sampling, see Figure~\ref{fig:teaser}.
Furthermore, as jump discontinuities are ubiquitous in low-level vision (e.g., caused by object edges, occlusion boundaries, changes in albedo, shadows, etc.), it is important 
to model them in a meaningful manner. By restricting either $\Omega$ or $\Gamma$ to a discrete set, one loses the ability to mathematically distinguish between continuous and discontinuous mappings.
\begin{figure}
  \captionsetup[subfloat]{justification=centering,singlelinecheck=false}
  \subfloat[]
  {
    \includegraphics[width=0.23\textwidth]{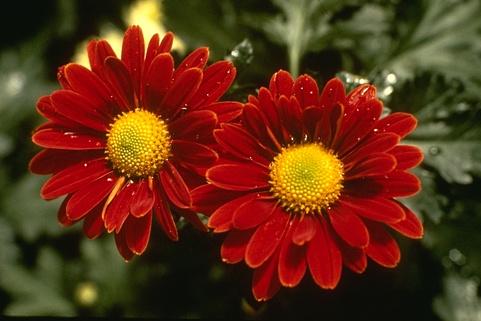}
  }
  \subfloat[]
  {
    \includegraphics[width=0.23\textwidth]{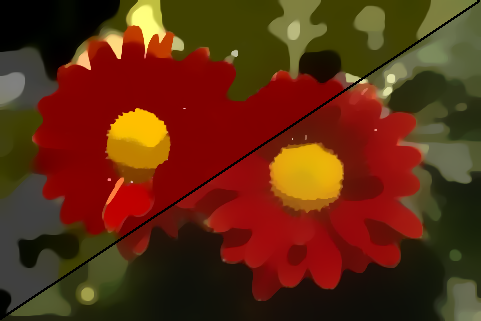}
  }%\\[-1mm]
  \caption{\label{fig:teaser} The classical way to discretize continuous convex relaxations such as 
    the vectorial Mumford-Shah functional \cite{Strekalovskiy-et-al-cvpr12} leads to solutions (\textbf{b)}, top-left) with
    a strong bias towards the chosen labels (here an equidistant $5 \times 5 \times 5$ sampling of the RGB space). This can be seen in the bottom left part of the image, where the green color is
    truncated to the nearest label which is gray. The proposed sublabel-accurate approximation of the continuous relaxation leads to bias-free solutions
    (\textbf{b)}, bottom-right).}
\end{figure}

Motivated by these two points we consider \emph{fully-continuous} optimization approaches, where the idea is to postpone the discretization of $\Omega \subset \bbR^n$ and $\Gamma \subset \bbR$
as long as possible. The prototypical class of continuous optimization problems which we consider in this work are nonconvex free-discontinuity problems, inspired 
by the celebrated Mumford-Shah functional \cite{Blake-Zisserman-87,MumShah}:
\begin{equation}
  \begin{aligned}
  E(u) = &\int_{\Omega \setminus J_u} f \left( x, u(x), \nabla u(x) \right) \mathrm{d}x \\
  + &\int_{J_u}d \left( x, u^-(x), u^+(x), \nu_u(x) \right) \mathrm{d} \mathcal{H}^{n-1}(x).
  \end{aligned}
  \label{eq:unlifted_cont_mshah_general}
\end{equation}
The first integral is defined on the region $\Omega \setminus J_u$ where $u$ is continuous. The integrand $f : \Omega \times \Gamma \times \bbR^n \to [0, \infty]$ 
can be thought of as a combined data term and regularizer, where the regularizer can penalize variations in terms of the 
(weak) gradient $\nabla u$. The second integral is defined on the $(n-1)$-dimensional discontinuity set $J_u \subset \Omega$ and 
$d : \Omega \times \Gamma \times \Gamma \times \Ss^{n-1} \to [0, \infty]$ penalizes jumps from $u^-$ to $u^+$ in unit direction $\nu_u$.
The appropriate function space for \eqref{eq:unlifted_cont_mshah_general} are the \emph{special functions of bounded variation}.
These are functions of bounded variation (cf. Section~\ref{sec:preliminaries} for a defintion) whose distributional derivative $Du$ can be decomposed 
into a continuous part and a jump part in the spirit of \eqref{eq:unlifted_cont_mshah_general}:
\begin{equation}
  Du = \nabla u \cdot \mathcal{L}^n + \left( u^+ - u^- \right) \nu_u \cdot \mathcal{H}^{n-1} \measurerestr J_u,
  \label{eq:decomposition}
\end{equation}
where $\mathcal{L}^n$ denotes the $n$-dimensional Lebesgue measure and $\mathcal{H}^{n-1} \measurerestr J_u$ the $(n-1)$-dimensional Hausdorff measure restricted to the jump set $J_u$.
For an introduction to functions of bounded variation and the study of existence of minimizers to \eqref{eq:unlifted_cont_mshah_general} we refer the interested reader to \cite{BV}. 

Note that due to the possible nonconvexity of $f$ in the first two variables a surprisingly large class of low-level vision problems fits the general framework 
of \eqref{eq:unlifted_cont_mshah_general}. While \eqref{eq:unlifted_cont_mshah_general} is a difficult nonconvex optimization problem, the state-of-the-art are convex relaxations
\cite{ABDM,bouchitte1998,ChJCA}. We give an overview of the idea behind the convex reformulation in
Section~\ref{sec:the_convex_relaxation}.% and Figure~\ref{fig:graph}.

Extensions to the vectorial setting, i.e., $\dim(\Gamma) > 1$, have been studied by Strekalovskiy \etal in various works
\cite{goldluecke2013tight, Strekalovskiy-et-al-cvpr12,strekalovskiy-et-al-siims14} and recently using the theory of currents by Windheuser~and~Cremers~\cite{windheuser2016convex}. 
The case when $\Gamma$ is a manifold has been considered by Lellmann \etal \cite{lellmann-et-al-iccv2013}. These advances have allowed for a wide range of difficult vectorial 
and joint optimization problems to be solved within a convex framework.

\subsection{Related work}
The first practical implementation of \eqref{eq:unlifted_cont_mshah_general} was proposed by Pock \etal \cite{PCBC-ICCV09}, using a simple finite differencing scheme in both $\Omega$ and
$\Gamma$ which has remained the standard way to discretize convex relaxations. This leads to a strong label bias (see Figure~\ref{fig:teaser}b), top-left) \emph{despite} the
initially label-continuous formulation. 

In the MRF community, a related approach to overcome this label-bias are \emph{discrete-continuous} models (discrete $\Omega$ and continuous $\Gamma$), pioneered by Zach~\etal
\cite{Zach-aistats13,Zach-Kohli-eccv12}. Most similar to the present work is the approach of Fix~and~Agarwal~\cite{fix2014duality}. They derive the discrete-continuous approaches as a discretization
of an infinite dimensional dual linear program. Their approach differs from ours, as we start from a different (nonlinear) infinite-dimensional optimization
problem and consider a representation of the dual variables which enforces continuity. 
The recent work of Bach \cite{bach2015submodular} extends the concept of submodularity from discrete to continuous $\Gamma$ along with complexity estimates.
 
There are also \emph{continuous-discrete} models, i.e. the range $\Gamma$ is discretized into labels but $\Omega$ is kept continuous
\cite{Chambolle-et-al-siims12,Lellmann-Schnoerr-siims11}. Recently, these spatially continuous multilabeling models have been extended to 
allow for so-called \emph{sublabel accurate} solutions \cite{laude16eccv,moellenhoff-laude-cvpr-2016}, i.e., solutions which lie between two labels. These are, however, limited to total variation regularization, due to the separate convexification of data term and regularizer. We show in this work that for general regularizers a joint convex relaxation is crucial.

Finally, while not focus of this work, there are of course also \emph{fully-discrete} approaches, among many \cite{Ishikawa,Schlesinger76,Werner-tpami2007}, which inspired some of the continuous formulations.

\subsection{Contribution}
In this work, we propose an approximation strategy for \emph{fully-continuous} relaxations which retains continuous $\Gamma$ 
even after discretization (see Figure~\ref{fig:teaser}b), bottom-right). We summarize 
our contributions as: 
\begin{itemize}
\item We generalize the work \cite{moellenhoff-laude-cvpr-2016} from total variation to general convex and nonconvex regularization.

\item We prove (see Prop.~\ref{prop:equiv_standard} and Prop.~\ref{prop:equiv_sublabel}) that different approximations to a convex relaxation of \eqref{eq:unlifted_cont_mshah_general} give rise to
  existing relaxations \cite{PCBC-ICCV09} and \cite{moellenhoff-laude-cvpr-2016}. We investigate the relationship to discrete-continuous MRFs in Prop.~\ref{prop:zach_equiv}.

\item On the example of the vectorial Mumford-Shah functional \cite{Strekalovskiy-et-al-cvpr12} we show that our framework yields also sublabel-accurate formulations of extensions to \eqref{eq:unlifted_cont_mshah_general}.
\end{itemize}

\section{Notation and preliminaries}
\label{sec:preliminaries}

We denote the Iverson bracket as $\iver{\cdot}$.
Indicator functions from convex analysis which take on values $0$ and $\infty$ are denoted by $\delta\{ \cdot \}$. We denote by $f^*$ the convex conjugate of $f : \bbR^n \to \bbRext$.
Let $\Omega \subset \bbR^n$ be a bounded open set. For a function $u \in L^1(\Omega; \bbR)$ its total variation is defined by 
\begin{equation}
TV(u) = \sup \left \{ \int_{\Omega} u \Div \varphi ~ \mathrm{d}x : \varphi
\in C_c^{1}(\Omega; \bbR^n) \right \}. 
\end{equation}
The space of functions of bounded variation, i.e., for which $\TV(u) < \infty$ (or equivalently for which the distributional derivative $Du$ is a finite Radon measure) is denoted by $\BV(\Omega; \bbR)$ \cite{BV}. 
We write $u \in \SBV(\Omega; \bbR)$ for functions $u \in \BV(\Omega; \bbR)$ whose distributional derivative admits the decomposition \eqref{eq:decomposition}.
For the rest of this work, we will make the following simplifying assumptions:
\begin{itemize}
  \item The Lagrangian $f$ in \eqref{eq:unlifted_cont_mshah_general}
 is separable, i.e., 
    \begin{equation}
      f(x, t, g) = \rho(x, t) + \eta(x, g),
    \end{equation}
    for possibly nonconvex $\rho : \Omega \times \Gamma \to \bbR$ 
    and regularizers $\eta : \Omega \times \bbR^n \to \bbR$ which are convex in $g$. 
  \item The jump regularizer in \eqref{eq:unlifted_cont_mshah_general}
 is isotropic and induced by a concave function $\kappa : \bbR_{\geq 0} \to \bbR$:
    \begin{equation}
      d(x, u^-, u^+, \nu_u) = \kappa( |u^- - u^+|) \norm{\nu_u}_2,
    \end{equation}
    with $\kappa(a) = 0 \Leftrightarrow a = 0$.

  \item The range $\Gamma = [\gamma_1, \gamma_\ell] \subset \bbR$ is a compact interval.
\end{itemize}

\section{The convex relaxation}
\figGraph
\label{sec:the_convex_relaxation}
In \cite{ABDM,bouchitte1998,ChJCA} the authors propose a convex relaxation for the problem \eqref{eq:unlifted_cont_mshah_general}.
Their basic idea is to reformulate the energy \eqref{eq:unlifted_cont_mshah_general} in terms of 
the \emph{complete graph} of $u$, i.e. lifting the problem to one dimension higher as illustrated 
in Figure~\ref{fig:graph}.
The complete graph $G_u \subset \Omega \times \Gamma$ is defined as the (measure-theoretic) boundary
of the characteristic function of the subgraph $\one{u} : \Omega \times \bbR \to \{0, 1\}$ given by:
\begin{align}
\one{u}(x,t) &= \iver{t < u(x)}.
\end{align}
Furthermore we denote the inner unit normal to $\one{u}$ with $\nu_{G_u}$.
It is shown in \cite{ABDM} that for $u \in \SBV(\Omega; \bbR)$ one has
\begin{equation}
  \begin{aligned}
    E(u) = F(\one{u}) &=  \sup_{\varphi \in \mathcal{K}} ~ \int_{G_u} \iprod{\varphi}{\nu_{G_u}} ~ \mathrm{d} \mathcal{H}^n,
  \end{aligned}
  \label{eq:lifted_cont_mshah_2}
\end{equation} 
with constraints on the dual variables $\varphi \in \mathcal{K}$ given by 
\begin{align}
  \mathcal{K} = \Bigl \{ &(\varphi_x, \varphi_t) \in C_c^1(\Omega
                           \times \bbR; \bbR^n \times \bbR):~ \nonumber\\
                         & \varphi_t(x, t) + \rho(x, t) \geq \eta^*(x, \varphi_x(x, t)), \label{eq:constraints_continuous} \\
                         & \bigl \| \int_{t}^{t'} \varphi_x(x, t) \mathrm{d}t \bigr \|_2 \leq \kappa(|t -  t'|),
                           \forall t, t', \forall x
                           \Bigr \}. \label{eq:constraints_jump}
\end{align}
The functional \eqref{eq:lifted_cont_mshah_2} can be interpreted as the maximum flux of admissible vector fields $\varphi \in \mathcal{K}$ through the cut given by the complete graph $G_u$. 
The set $\mathcal{K}$ can be seen as 
capacity constraints on the flux field $\varphi$. This is reminiscent to constructions from the
discrete optimization community \cite{Ishikawa}. The constraints \eqref{eq:constraints_continuous} correspond to the first integral in \eqref{eq:unlifted_cont_mshah_general} and the non-local constraints \eqref{eq:constraints_jump} to the jump penalization.

Using the fact that the distributional derivative of the subgraph indicator function $\one{u}$ can be written as 
\begin{equation}
  D \one{u} = \nu_{G_u} \cdot \mathcal{H}^m \measurerestr G_u,
\end{equation} 
one can rewrite the energy \eqref{eq:lifted_cont_mshah_2} as
\begin{equation}
  \begin{aligned}
    F(\one{u}) &=  \sup_{\varphi \in \mathcal{K}} ~ \int_{\Omega \times \Gamma} \iprod{\varphi}{D \one{u}}.
  \end{aligned}
  \label{eq:lifted_cont_mshah}
\end{equation} 
A convex formulation is then obtained by relaxing the set of admissible primal variables to a convex set:
\begin{equation}
  \begin{aligned}
  \mathcal{C} = \Bigl \{ &v \in \BV_{\text{loc}}(\Omega \times \bbR; [0,1]) :~ \\
  &v(x, t) = 1 ~~ \forall t \leq \gamma_1, v(x,t) = 0 ~~ \forall t > \gamma_\ell, \\
  &v(x, \cdot) ~ \text{non-increasing} \Bigr \}.
  \end{aligned}
  \label{eq:primal_constraints}
\end{equation}
This set can be thought of as the convex hull of the subgraph functions $\one{u}$. The final optimization problem
is then a convex-concave saddle point problem given by:
\begin{equation}
  \inf_{v \in \mathcal{C}} ~ \sup_{\varphi \in \mathcal{K}} ~ \int_{\Omega \times \bbR} \iprod{\varphi}{Dv}.
  \label{eq:lifted_relaxed_cont_mshah}
\end{equation}
In dimension one ($n = 1$), this convex relaxation is tight \cite{carioni2016discrete,ChJCA}. For $n > 1$ global optimality can be guaranteed by means of a thresholding theorem in case $\kappa \equiv \infty$ \cite{bouchitte2015duality,PCBC-SIIMS}.
If the primal solution $\widehat v \in \mathcal{C}$ to \eqref{eq:lifted_relaxed_cont_mshah} is binary, the global optimum $u^*$ of \eqref{eq:unlifted_cont_mshah_general} can be recovered simply by pointwise thresholding
$\widehat u(x) = \sup \{ t : \widehat v(x, t) > \frac{1}{2} \}$. If $\widehat v$ is not binary, in the general setting it is not clear how to obtain the global optimal solution from the relaxed
solution. An a posteriori optimality bound to the global optimum $E(u^*)$ of \eqref{eq:unlifted_cont_mshah_general} for the thresholded solution $\widehat u$ can be computed by:
\begin{equation}
  | E(\widehat u) - E(u^*) | \leq | F(\one{\widehat u}) - F(\widehat v) |.
\end{equation}
Using that bound, it has been observed that solutions are usually near globally optimal \cite{Strekalovskiy-et-al-cvpr12}.
In the following section, we show how different discretizations of 
the continuous problem \eqref{eq:lifted_relaxed_cont_mshah} lead to various 
existing lifting approaches and to generalizations of the recent sublabel-accurate
continuous multilabeling approach \cite{moellenhoff-laude-cvpr-2016}.

\section{Sublabel-accurate discretization}
\label{sec:sublabel_disc}
\figFEM
\figConstantVsLinear
\subsection{Choice of primal and dual mesh}
In order to discretize the relaxation
\eqref{eq:lifted_relaxed_cont_mshah}, we partition the range $\Gamma = [\gamma_1, \gamma_\ell]$ into $k = \ell - 1$ intervals. 
The individual intervals $\Gamma_i = [\gamma_i,
\gamma_{i+1}]$ form a one dimensional \emph{simplicial complex} (see e.g.,~\cite{Hirani2003}), and we have $\Gamma = \Gamma_1 \cup \hdots \cup \Gamma_k$.
The points $\gamma_i \in \Gamma$ are also
referred to as \emph{labels}. We assume that the
labels are equidistantly spaced with label distance $h = \gamma_{i+1} -
\gamma_i$. The theory generalizes also to non-uniformly spaced labels, as long as the spacing is homogeneous in $\Omega$. 
Furthermore, we define $\gamma_0 = \gamma_1 - h$ and $\gamma_{\ell+1} = \gamma_\ell + h$.

The mesh for dual variables is given by \emph{dual complex}, which is formed by the intervals $\Gamma_i^*
= [\gamma_{i-1}^*, \gamma_{i}^*]$ with nodes $\gamma^*_i
= \frac{\gamma_{i} + \gamma_{i+1}}{2}$. An overview of the notation and the
considered finite dimensional approximations is given in Figure~\ref{fig:fem}.

\subsection{Representation of the primal variable}
As $\one{u}$ is a discontinuous jump function, we consider a piecewise
constant approximation for $v \in \mathcal{C}$,
\begin{equation}
  \Phi^0_i(t) = \iver{ t \in \Gamma_i }, ~ 1 \leq i \leq k,
\end{equation}
see Figure~\ref{fig:fem}a). Due to the boundary conditions in Eq.~\eqref{eq:primal_constraints}, we set $v$ outside of $\Gamma$
to $1$ on the left and $0$ on the right. Note that the non-decreasing constraint in $\mathcal{C}$ is implicitly realized as $\varphi_t \in \mathcal{K}$ can be arbitrarily large. 

For coefficients $\hat v : \Omega \times \{ 1, \hdots, k \} \to \bbR$ we have
\begin{equation}
  v(x,t) = \sum_{i=1}^k \hat v(x,i) \Phi^0_i(t).
  \label{eq:ansatz_v}
\end{equation}
As an example of this representation, consider the approximation of $\one{u}$ at point $p$ shown in Figure~\ref{fig:graph}:
\begin{equation}
  \begin{aligned}
    \widehat v(p, \cdot) &= \sum_{i=1}^k e_i \int_{\Gamma} \Phi^0_i(t) \one{u}(p, t) \mathrm{d}t \\
    &= h \cdot \begin{bmatrix} 1 & 1 & 0.4 & 0 \end{bmatrix}^\top.
  \end{aligned}
  \label{eq:sublabel_inter}
\end{equation}
This leads to the sublabel-accurate representation also considered in \cite{moellenhoff-laude-cvpr-2016}. In that work, the representation from the above example \eqref{eq:sublabel_inter} encodes a convex combination
between the labels $\gamma_3$ and $\gamma_4$ with interpolation factor $0.4$. 
Here it is motivated from a different perspective: 
we take a finite dimensional subspace approximation of the infinite dimensional optimization problem \eqref{eq:lifted_relaxed_cont_mshah}.

\subsection{Representation of the dual variables}
\subsubsection{Piecewise constant $\varphi_t$}
The simplest discretization of the dual variable $\varphi_t$ is
to pick a piecewise constant approximation on the dual intervals $\Gamma_i^*$
as shown in Figure~\ref{fig:fem}b):
The functions are given by 
\begin{equation}
  \Psi^0_i(t) = \iver{ t \in \Gamma_i^* }, ~ 1 \leq i \leq \ell,
\end{equation}
As $\varphi$ is a vector
field in $C_c^1$, the functions $\Psi$ vanish
outside of $\Gamma$. For coefficient functions $\hat \varphi_t : \Omega \times \{1, \hdots, \ell\} \to \bbR$ and $\hat \varphi_x : \Omega \times \{1, \hdots, k\} \to \bbR^n$ we have:
\begin{equation}
  \varphi_t(t) = \sum_{i=1}^\ell \hat \varphi_t(i) \Psi^0_i(t), ~ \varphi_x(t) = \sum_{i=1}^k \hat \varphi_x(i) \Phi^0_i(t).
  \label{eq:ansatz_phi_0}
\end{equation}
To avoid notational clutter, we dropped $x \in \Omega$ in \eqref{eq:ansatz_phi_0} and will do 
so also in the following derivations. Note that for $\varphi_x$ we chose the same piecewise constant approximation as for $v$, 
as we keep the model continuous in $\Omega$, and ultimately
discretize it using finite differences in $x$. 

\paragraph{Discretization of the constraints}
 In the following, we will plug in the finite dimensional approximations into the constraints from the set
$\mathcal{K}$.
We start by reformulating the constraints in \eqref{eq:constraints_continuous}.
Taking the infimum over $t \in \Gamma_i$ they can be equivalently written
as:
\begin{equation}
  \inf_{t \in \Gamma_i} ~ \varphi_t(t) + \rho(t) - \eta^* \left( \varphi_x(t) \right) \geq 0, ~ 1 \leq i \leq \ell.
  \label{eq:separable_constraints_infimum}
\end{equation}
Plugging in the approximation \eqref{eq:ansatz_phi_0}
into the above leads to the following constraints for $1 \leq i \leq k$:
\begin{equation}
  \begin{aligned}
    \hat \varphi_t(i) + &\inf_{t \in [\gamma_i, \gamma_i^*]} \rho(t) \geq 
  \eta^*(\hat \varphi_x(i)), \\ 
    \hat \varphi_t(i + 1) + &\underbrace{\inf_{t \in [\gamma_i^*, \gamma_{i+1}]} \rho(t)}_{\text{min-pooling}} \geq 
  \eta^*(\hat \varphi_x(i)).
   \end{aligned}
  \label{eq:discretized_separable_constraints_00}
\end{equation}
These constraints can be seen as min-pooling of the continuous unary potentials in a symmetric region centered on the 
label $\gamma_i$. To see that more easily, assume one-homogeneous regularization so that $\eta^* \equiv 0$ on its domain. Then two consecutive constraints from \eqref{eq:discretized_separable_constraints_00} can be combined into one where the infimum of $\rho$ is taken over $\Gamma_i^* = [\gamma_i^*, \gamma_{i+1}^*]$ centered the label $\gamma_i$. This leads to capacity constraints for the flow in vertical direction $-\hat \varphi_t(i)$ of the form 
\begin{equation}
  -\hat \varphi_t(i) \leq \inf_{t \in \Gamma_i^*} \rho(t), ~ 2 \leq i \leq \ell - 1,
\end{equation}
as well as similar constraints on $\hat \varphi_t(1)$ and $\hat \varphi_t(\ell)$.
The effect of this on a nonconvex energy is shown in Figure~\ref{fig:constant_vs_linear} on the left.
The constraints \eqref{eq:discretized_separable_constraints_00} are convex inequality constraints, which can be implemented
using standard proximal optimization methods and
orthogonal projections onto the epigraph $\epi(\eta^*)$ as described in \cite[Section~5.3]{PCBC-SIIMS}.

For the second part of the constraint set
\eqref{eq:constraints_jump} we insert again the finite-dimensional
representation \eqref{eq:ansatz_phi_0}  to arrive at:
\begin{equation}
  \begin{aligned}
    &\bigl \| (1-\alpha) \hat \varphi_x(i) + \sum_{l=i+1}^{j-1} \hat
    \varphi_x(l) + \beta \hat \varphi_x(j) \bigr \| \\
    &\quad \leq \frac{\kappa( \gamma_j^\beta
 - \gamma_i^\alpha)}{h}, ~ \forall \, 1 \leq i \leq j \leq k, \alpha, \beta \in [0,1],
  \end{aligned}
  \label{eq:infinite_jump_constraints}
\end{equation}
where $\gamma_i^\alpha := (1-\alpha)\gamma_i + \alpha
\gamma_{i+1}$. 
These are infinitely many constraints, but similar to \cite{moellenhoff-laude-cvpr-2016}
these can be implemented with finitely many constraints.
\begin{prop}
  For concave $\kappa : \bbR^+_0 \to \bbR$ with $\kappa(a)=0 \Leftrightarrow a = 0$, the constraints \eqref{eq:infinite_jump_constraints}
  are equivalent to
  \begin{equation}
    \bigl \| \sum_{l=i}^j \hat \varphi_x(l) \bigr \| \leq \frac{\kappa(\gamma_{j+1} - \gamma_i)}{h}, ~ \forall 1 \leq i \leq j \leq k.
    \label{eq:kappa_constraints}
  \end{equation}
   \label{prop:kappa_constraints}
\end{prop}
\begin{proof}
  Proofs are given in the appendix.
\end{proof}
This proposition reveals that only information from the labels $\gamma_i$ enters 
into the jump regularizer $\kappa$. For $\ell=2$ we expect all regularizers to behave like the total variation.

\paragraph{Discretization of the energy}
For the discretization of the saddle point energy \eqref{eq:lifted_relaxed_cont_mshah}
we apply the divergence theorem 
\begin{equation}
  \int_{\Omega \times \bbR} \iprod{\varphi}{Dv} = \int_{\Omega \times \bbR} -\Div \varphi \cdot v ~ \mathrm{d}t ~ \mathrm{d}x,
  \label{eq:lifted_relaxed_cont_mshah_weak}
\end{equation}
and then discretize the divergence by inserting the piecewise constant representations of
$\varphi_t$ and $v$:
\begin{equation}
  \begin{aligned}
    &\int_{\bbR} -\partial_t \varphi_t(t) v(t) ~ \mathrm{d}t =\\
    &-\hat \varphi_t(1) - \sum_{i=1}^k \hat v(i) \left[\hat \varphi_t(i + 1) -
      \hat \varphi_t(i) \right].
  \end{aligned}
  \label{eq:div_phit}
\end{equation}
The discretization of the other parts of the divergence are given as
the following:
\begin{equation}
  \begin{aligned}
    &\int_{\bbR} -\partial_{x_j} \varphi_x(t) v(t) ~ \mathrm{d}t = 
    -h \sum_{i=1}^k \partial_{x_j} \hat \varphi_x(i) \hat v(i),
  \end{aligned}
  \label{eq:div_phix}
\end{equation}
where the spatial derivatives
$\partial_{x_j}$ are ultimately discretized using standard finite differences.
It turns out that the above discretization can be related to the one from \cite{PCBC-ICCV09}:
\begin{prop}
For convex one-homogeneous $\eta$ the discretization with piecewise constant $\varphi_t$ and $\varphi_x$ leads 
to the traditional discretization as proposed in \cite{PCBC-ICCV09}, except with min-pooled instead of sampled unaries. 
\label{prop:equiv_standard}
\end{prop}

\subsubsection{Piecewise linear $\varphi_t$}
As the dual variables in $\mathcal{K}$ are continuous vector fields, 
a more faithful approximation is given by a continuous piecewise linear approximation,
given for $1 \leq i \leq \ell$ as:
\begin{equation}
  \Psi_i^1(t) = 
  \begin{cases}
    \frac{t - \gamma_{i-1}}{h}, &\text{ if } t \in [\gamma_{i-1}, \gamma_{i}],\\
    \frac{\gamma_{i+1} - t}{h}, &\text{ if } t \in [\gamma_{i}, \gamma_{i + 1}],\\
    0 &\text{ otherwise.}
  \end{cases}
  \label{eq:pw_lin_basis}
\end{equation}
They are shown in Figure~\ref{fig:fem}c), and we set:
\begin{equation}
  \varphi_t(t) = \sum_{i=1}^\ell \hat \varphi_t(i) \Psi_i^1(t).
  \label{eq:piecw_linear_phit}
\end{equation}
Note that the piecewise linear 
dual representation considered by Fix~\etal in \cite{fix2014duality} differs in this point,
as they do not ensure a continuous representation. Unlike the proposed approach their approximation does not take 
a true subspace of the original infinite dimensional function space.

\paragraph{Discretization of the constraints}
We start from the reformulation \eqref{eq:separable_constraints_infimum} of the original constraints \eqref{eq:constraints_continuous}.
With \eqref{eq:piecw_linear_phit}
for $\varphi_t$ and \eqref{eq:ansatz_phi_0} for $\varphi_x$, we have for $1 \leq i \leq k$:
\begin{equation}
  \begin{aligned}
    \inf_{t \in \Gamma_i} ~&\hat \varphi_t(i) \frac{\gamma_{i+1} -
      t}{h} + \hat \varphi_t(i+1)
    \frac{t - \gamma_{i}}{h} \\
    &+ \rho(t) \geq \eta^*(\hat \varphi_x(i)).
  \end{aligned}
  \label{eq:constraints_sublabel}
\end{equation}
While the constraints \eqref{eq:constraints_sublabel} seem difficult to implement,
they can be reformulated in a simpler way involving $\rho^*$.
\begin{prop}
The constraints \eqref{eq:constraints_sublabel} can be equivalently reformulated by 
introducing additional variables $a \in \bbR^k$, $b \in \bbR^k$, where $\forall i \in \{ 1,  \hdots, k \}$:
\begin{equation}
  \begin{aligned}
 &r(i) = (\hat \varphi_t(i) - \hat \varphi_t(i+1)) / h,\\
 &a(i) + b(i) - (\hat \varphi_t(i) \gamma_{i+1} - \hat \varphi_t(x, i+1)
   \gamma_{i}) / h = 0,\\ 
   &r(i) \geq \rho_i^* \left( a(i) \right), \hat \varphi_x(i) \geq \eta^* \left( b(i) \right),
  \end{aligned}
  \label{eq:piecw_lin_constraints}
\end{equation}
with $\rho_i(x, t) = \rho(x, t) + \delta\{ t \in \Gamma_i \}$.
\end{prop}
The constraints \eqref{eq:piecw_lin_constraints} are implemented by projections onto the epigraphs of $\eta^*$ and $\rho_i^*$, as they can be written as:
\begin{equation}
    (r(i), a(i)) \in \epi (\rho_i^*),~ (\hat \varphi_x(i), b(i)) \in \epi(\eta^*).
\end{equation}
Epigraphical projections for quadratic and piecewise linear $\rho_i$ are described in
\cite{moellenhoff-laude-cvpr-2016}. In Section~\ref{sec:piecw_quad} we describe how to 
implement piecewise quadratic $\rho_i$. 
As the convex conjugate of $\rho_i$ enters into the constraints, it becomes
clear that this discretization only sees the \emph{convexified} unaries on each
interval, see also the right part of Figure~\ref{fig:constant_vs_linear}.

\paragraph{Discretization of the energy} It turns out that the piecewise linear representation of $\varphi_t$ leads
to the same discrete bilinear saddle point term as \eqref{eq:div_phit}. The other term remains unchanged, as we pick the 
same representation of $\varphi_x$.

\paragraph{Relation to existing approaches} In the following we point out the relationship of the approximation with piecewise linear $\varphi_t$ to
the sublabel-accurate multilabeling approaches \cite{moellenhoff-laude-cvpr-2016} and the discrete-continuous MRFs \cite{Zach-Kohli-eccv12}.
\begin{prop}
The discretization with piecewise linear $\varphi_t$ and piecewise constant $\varphi_x$, together with the 
choice $\eta(g) = \norm{g}$ and $\kappa(a) = a$ is equivalent to the relaxation \cite{moellenhoff-laude-cvpr-2016}.
\label{prop:equiv_sublabel}
\end{prop}
Thus we extend the relaxation proposed in \cite{moellenhoff-laude-cvpr-2016} to more general regularizations. 
The relaxation \cite{moellenhoff-laude-cvpr-2016} was derived starting from a discrete label space
and involved a separate relaxation of data term and regularizer. To see this, first note that the convex conjugate 
of a convex one-homogeneous function is the indicator function of a 
convex set \cite[Corollary~13.2.1]{Rockafellar:ConvexAnalysis}. Then the constraints \eqref{eq:constraints_continuous} can be written as
\begin{align}
  -\varphi_t(x, t) &\leq \rho(x, t), \label{eq:split_dataterm} \\
  \varphi_x(x, t) &\in \dom \{ \eta^* \}, \label{eq:split_regularizer}
\end{align}
where \eqref{eq:split_dataterm} is the data term and \eqref{eq:split_regularizer} the regularizer.
This provides an intuition why the separate convex relaxation of data term and regularizer in \cite{moellenhoff-laude-cvpr-2016} worked well. However, for general choices of $\eta$ a joint relaxation of data term and regularizer as in \eqref{eq:constraints_sublabel} is crucial. 
The next proposition establishes the relationship between the data term from \cite{Zach-Kohli-eccv12} and the one from \cite{moellenhoff-laude-cvpr-2016}.
\begin{prop}
  The data term from \cite{moellenhoff-laude-cvpr-2016} (which is in turn a special case of the discretization with piecewise linear $\varphi_t$) can be (pointwise) brought into the primal form
  \begin{equation}
    \mathcal{D}(\widehat v) = \inf_{\substack{x_i \geq 0,\sum_i x_i=1\\\widehat v = y / h + I^\top x}} ~ \sum_{i=1}^k x_i \rho_i^{**} \left(\frac{y_i}{x_i} \right),
    \label{eq:dataterm_zach}
  \end{equation}
  where $I \in \bbR^{k \times k}$ is a discretized integration operator. 
  \label{prop:zach_equiv}
\end{prop}
 The data term of Zach and Kohli \cite{Zach-Kohli-eccv12} is precisely given by \eqref{eq:dataterm_zach} except that the optimization is directly performed on  $x,y \in \bbR^k$. The
variable $x$ can be interpreted as 1-sparse indicator of the interval $\Gamma_i$ and $y \in \bbR^k$ as a sublabel offset. The constraint $\widehat v = y / h + I^\top x$ connects this representation to
the subgraph representation $\widehat v$ via the operator $I \in \bbR^{k \times k}$ (see appendix for the definition). For general regularizers $\eta$, the discretization with piecewise
linear $\varphi_t$ differs from \cite{moellenhoff-laude-cvpr-2016} as we perform a \emph{joint convexification} of data term and regularizer and from \cite{Zach-Kohli-eccv12} as we consider the spatially continuous setting.
Another important question to ask is which primal formulation is actually optimized 
after discretization with piecewise linear $\varphi_t$. In particular the distinction 
between jump and smooth regularization only makes sense for continuous label spaces, so 
it is interesting to see what is optimized after discretizing the label space.
\begin{prop}
  Let $\gamma  = \kappa(\gamma_2 - \gamma_1)$ and $\ell = 2$. The approximation with piecewise linear $\varphi_t$ and piecewise constant $\varphi_x$ of the continuous optimization problem \eqref{eq:lifted_relaxed_cont_mshah} is equivalent to
  \begin{equation}
    \inf_{u : \Omega \to \Gamma} \int_{\Omega} \rho^{**}(x, u(x)) + (\eta^{**}
    ~ \square ~ \gamma \norm{\cdot})
  (\nabla u(x)) ~\mathrm{d}x,
  \label{eq:unlifted_prob}
  \end{equation}
  where $(\eta ~ \square ~ \gamma \norm{\cdot})(x) = \inf_{y} ~ \eta(x - y) + \gamma \norm{y}$ denotes the infimal convolution (cf. \cite[Section~5]{Rockafellar:ConvexAnalysis}).
  \label{prop:infconv}
\end{prop}
From Proposition~\ref{prop:infconv} we see that the minimal discretization with $\ell=2$ amounts to approximating problem \eqref{eq:unlifted_cont_mshah_general} by globally convexifying the data term. Furthermore, we can see that Mumford-Shah (truncated quadratic) regularization ($\eta(g) = \alpha \norm{g}^2$, $\kappa(a) \equiv \lambda \iver{a > 0}$) is approximated by a convex Huber regularizer in case $\ell = 2$. This is because the infimal convolution between $x^2$ and $|x|$ corresponds to the Huber function. While even for $\ell = 2$ this is a reasonable approximation to the original model \eqref{eq:unlifted_cont_mshah_general}, we can gradually increase the number of labels to get an increasingly faithful approximation of the original nonconvex problem.

\subsubsection{Piecewise quadratic $\varphi_t$}
For piecewise quadratic $\varphi_t$ the main difficulty are the constraints 
in \eqref{eq:separable_constraints_infimum}. For piecewise linear $\varphi_t$
the infimum over a linear function plus $\rho_i$ lead to (minus) the convex
conjugate of $\rho_i$. Quadratic dual variables lead to so 
called generalized $\Phi$-conjugates \cite[Chapter~11L*,~Example~11.66]{VariAna}. 
Such conjugates were also theoretically considered in the recent work \cite{fix2014duality} for discrete-continuous MRFs, however an efficient implementation seems challenging. The advantage of this representation would be that one
can avoid convexification of the unaries on each interval $\Gamma_i$ and thus obtain a tighter approximation.
While in principle the resulting constraints could be implemented using techniques from convex algebraic
geometry  and semi-definite programming \cite{Blekherman} we leave this direction open
to future work.

\section{Implementation and extensions}
\figConvexExact
\figJointStereoSegm
\subsection{Piecewise quadratic unaries $\rho_i$}
\label{sec:piecw_quad}
In some applications such as robust fusion of depth maps, the data term $\rho$ has a
piecewise quadratic form:
\begin{equation}
  \rho(u) = \sum_{m=1}^M \min \left\{ \nu_m, \alpha_m \left( u - f_m
    \right)^2 \right \}.
  \label{eq:sum_robust_dt}
\end{equation}
The intervals on which the above function is a quadratic are formed by the breakpoints $f_m \pm \sqrt{\nu_m/\alpha_m}$.
In order to optimize this within our framework, we need to compute the
convex conjugate of $\rho$ on the intervals $\Gamma_i$, see
Eq.~\eqref{eq:piecw_lin_constraints}. 
We can write the data term \eqref{eq:sum_robust_dt} on each $\Gamma_i$ as
\begin{equation}
  \min_{1 \leq j \leq n_i} ~ \underbrace{a_{i,j} u^2 + b_{i,j} u + c_{i,j} + \delta \{ u \in I_{i,j} \}}_{=:\rho_{i,j}(u)},
\end{equation}
where $n_i$ denotes the number of pieces and the intervals $I_{i,j}$
are given by the breakpoints and $\Gamma_i$. The convex
conjugate is then given by
$\rho_i^*(v) = \max_{1 \leq j \leq n_i} ~ \rho_{i,j}^*(v)$.
As the epigraph of the maximum is the intersection of
the epigraphs, $\epi(\rho_i^*) = \bigcap_{j=1}^{n_j} ~ \epi \left( \rho_{i,j}^* \right)$,
the constraints for the data term
$  (r^i, a^i) \in \epi(\rho_i^*)$,
can be broken down: 
\begin{equation}
  \begin{aligned}
    &(r^{i,j}, a^{i,j}) \in \epi\left( \rho_{i,j}^* \right),
    r^i = r^{i,j}, a^i = a^{i,j}, \forall j.
  \end{aligned}
\end{equation}
The projection onto the epigraphs of the $\rho_{i,j}^*$ are carried out as
described in \cite{moellenhoff-laude-cvpr-2016}.
Such a convexified piecewise quadratic function is shown
on the right in Figure~\ref{fig:constant_vs_linear}.

\figDenoiseSynthetic
\figDenoiseSyntheticii
\subsection{The vectorial Mumford-Shah functional}
\label{sec:vecmshah}
Recently, the free-discontinuity problem \eqref{eq:unlifted_cont_mshah_general} has been 
generalized to vectorial functions $u : \Omega \to \bbR^{n_c}$ by Strekalovskiy \etal~\cite{Strekalovskiy-et-al-cvpr12}. The model they propose is
\begin{equation}
  \sum_{c=1}^{n_c} \int_{\Omega \setminus J_u} f_c(x, u_c(x), \nabla_x u_c(x)) \, \mathrm{d}x + \lambda \mathcal{H}^{n-1}(J_u),
\end{equation}
which consists of a separable data term and separable regularization on the continuous part. The individual channels are coupled through
the jump part regularizer $\mathcal{H}^{n-1}(J_u)$ of the joint jump set across all channels. 
Using the same strategy as in Section~\ref{sec:sublabel_disc}, applied to the relaxation described in \cite[Section~3]{Strekalovskiy-et-al-cvpr12}, a sublabel-accurate representation of the vectorial Mumford-Shah functional can be obtained.

\subsection{Numerical solution}
We solve the final finite dimensional optimization problem after finite-difference discretization in spatial direction using the primal-dual algorithm \cite{PCBC-ICCV09} implemented in the convex
optimization framework {\tt prost} \footnote{\url{https://github.com/tum-vision/prost}}. 

\section{Experiments}
\subsection{Exactness in the convex case}
We validate our discretization in Figure~\ref{fig:convex_exact} on the convex problem
$\rho(u) = (u - f)^2$, $\eta(\nabla u) = \lambda \normc{\nabla u}^2$.
The global minimizer of the problem is obtained by solving $(I - \lambda \Delta)u = f$. For piecewise linear $\varphi_t$ we recover the exact solution using only $2$ labels, and remain (experimentally) exact as we increase the number of labels. The discretization from \cite{PCBC-SIIMS} shows a strong label bias due to the piecewise constant dual variable $\varphi_t$. Even with $16$ labels their solution is different from the ground truth energy.

\subsection{The vectorial Mumford-Shah functional}
\paragraph{Joint depth fusion and segmentation}
We consider the problem of joint image segmentation and
robust depth fusion from~\cite{PZB07} using the vectorial Mumford-Shah functional from 
Section~\ref{sec:vecmshah}. The data term for the depth channel is given by
\eqref{eq:sum_robust_dt}, where $f_m$ are the input depth
hypotheses, $\alpha_m$ is a depth confidence and $\nu_m$ is a
truncation parameter to be robust towards outliers. For the
segmentation, we use a quadratic difference dataterm in RGB space.
For Figure~\ref{fig:joint_stereo_segm} we computed
multiple depth hypotheses $f_m$ on a stereo pair using different matching costs (sum of absolute (gradient) differences, and normalized cross correlation)
with varying patch radii ($0$ to $2$). Even for a moderate label
space of $5 \times 5 \times 5 \times 5$ we have no label discretization artifacts.

The piecewise linear approximation of the unaries in \cite{Strekalovskiy-et-al-cvpr12} leads to an  
almost piecewise constant segmentation of the image. To highlight the sublabel-accuracy of the proposed 
approach we chose a small smoothness parameter which leads to a piecewise smooth segmentation, but 
with a higher smoothness term or different choice of unaries a piecewise constant segmentation could also 
be obtained.

\paragraph{Piecewise-smooth approximations}
In Figure~\ref{fig:denoise_synthetic} we compare the discretizations for
the vectorial Mumford-Shah functional. We see that the
approach \cite{Strekalovskiy-et-al-cvpr12} shows strong label bias (see also Figure~\ref{fig:denoise_syntheticii}~and~\ref{fig:teaser}) while the discretiziation with piecewise linear duals leads to a sublabel-accurate result.

\section{Conclusion}
We proposed a framework to numerically solve \emph{fully-continuous} convex relaxations
in a sublabel-accurate fashion.  The key idea is to implement the 
dual variables using a piecewise linear approximation. 
We prove that different choices of approximations
for the dual variables give rise to various existing relaxations: in
particular piecewise constant duals lead to the traditional
lifting \cite{PCBC-ICCV09} (with min-pooling of the unary costs),
whereas piecewise linear duals lead to the sublabel lifting
that was recently proposed for total variation regularized problems
\cite{moellenhoff-laude-cvpr-2016}.  While the latter method is not
easily generalized to other regularizers due to the separate convexification
of data term and regularizer, the proposed representation generalizes
to arbitrary convex and non-convex regularizers such as the scalar and the
vectorial Mumford-Shah problem. 
The proposed approach provides a systematic technique to derive
sublabel-accurate discretizations for continuous convex relaxation
approaches, thereby boosting their memory and runtime efficiency for
challenging large-scale applications.

\section*{Appendix}
\setcounter{prop}{0}
\begin{prop}
  For concave $\kappa : \bbR^+_0 \to \bbR$ with $\kappa(a)=0 \Leftrightarrow a = 0$, the constraints 
  \begin{equation}
    \begin{aligned}
      &\bigl \| (1-\alpha) \hat \varphi_x(i) + \sum_{l=i+1}^{j-1} \hat \varphi_x(l) + \beta \hat \varphi_x(j) \bigr \| \\
      &\quad \leq \frac{\kappa( \gamma_j^\beta
        - \gamma_i^\alpha)}{h}, ~ \forall 1 \leq i \leq j \leq k, \alpha, \beta \in [0,1],
    \end{aligned}
    \label{app_eq:infinite_jump_constraints}
  \end{equation}
  are equivalent to
  \begin{equation}
    \bigl \| \sum_{l=i}^j \hat \varphi_x(l) \bigr \| \leq \frac{\kappa(\gamma_{j+1} - \gamma_i)}{h}, \forall 1 \leq i \leq j \leq k.
    \label{app_eq:kappa_constraints}
  \end{equation}
\end{prop}
\begin{proof}
The implication \eqref{app_eq:infinite_jump_constraints} $\Rightarrow$ \eqref{app_eq:kappa_constraints} clearly holds. Let us now assume the constraints \eqref{app_eq:kappa_constraints} are fulfilled.
% \begin{equation}
%   \norm{ \sum_{l=i}^{j} \varphi_x(l)} \leq
% \frac{ \kappa( \gamma_{j + 1} - \gamma_i )}{h}, \forall 1 \leq
% i \leq j \leq k.
% \tag{*}
% \end{equation}
First we show that the constraints \eqref{app_eq:infinite_jump_constraints} also hold for $\alpha \in [0,1]$ and $\beta \in \{0, 1\}$. First, we start with $\beta = 0$:
\begin{equation}
  \begin{aligned}
    &\norm{(1-\alpha)\hat \varphi_x(i) + \sum_{l=i+1}^{j-1} \hat \varphi_x(l)} = \\
    &\norm{(1 - \alpha) \sum_{l=i}^{j-1} \hat \varphi_x(l) + \alpha \sum_{l=i+1}^{j-1} \hat \varphi_x(l)} \leq\\
    &(1-\alpha) \norm{\sum_{l=i}^{j-1} \hat \varphi_x(l)} + \alpha \norm{\sum_{l=i+1}^{j-1} \hat \varphi_x(l)} \overset{\text{by}~\eqref{app_eq:kappa_constraints}}\leq\\
    & (1-\alpha) \frac{1}{h}\kappa(\gamma_j - \gamma_i) + \alpha \frac{1}{h} \kappa(\gamma_j - \gamma_{i+1})\overset{\text{concavity}}\leq\\
    &\frac{1}{h}\left( \kappa((1-\alpha) (\gamma_j - \gamma_i) + \alpha (\gamma_j - \gamma_{i+1}) \right) = 
    \frac{1}{h} \kappa(\gamma_j^0 - \gamma_i^{\alpha}).
  \end{aligned}
  \label{app_eq:beta0}
\end{equation}
In the same way, it can be
shown that for $\beta = 1$ we have:
\begin{equation}
  \begin{aligned}
    &\norm{(1-\alpha)\hat \varphi_x(i) + \sum_{l=i+1}^{j-1} \hat \varphi_x(l) + 1 \cdot \hat \varphi_x(j)} \leq 
    % &\norm{(1 - \alpha) \sum_{l=i}^{j} \hat \varphi_x(l) + \alpha \sum_{l=i+1}^{j} \hat \varphi_x(l)} \leq\\
    % &(1-\alpha) \norm{\sum_{l=i}^{j} \hat \varphi_x(l)} + \alpha \norm{\sum_{l=i+1}^{j} \hat \varphi_x(l)} \overset{\text{by}~(*)}\leq\\
    % & (1-\alpha) \frac{1}{h}\kappa(\gamma_{j+1} - \gamma_i) + \alpha \frac{1}{h} \kappa(\gamma_{j+1} - \gamma_{i+1})\overset{\text{concavity}}\leq \\
    % & \frac{1}{h} \left( \kappa((1-\alpha) (\gamma_{j+1} - \gamma_i) + \alpha (\gamma_{j+1} - \gamma_{i+1}) \right)=
    \frac{1}{h} \kappa(\gamma_j^1 - \gamma_i^{\alpha}).
  \end{aligned}
  \label{app_eq:beta1}
\end{equation}
We have shown the constraints to hold for $\alpha \in [0,1]$ and $\beta \in \{0, 1\}$. Finally we show they also hold for $\beta \in [0,1]$:
\begin{equation}
  \begin{aligned}
    &\norm{(1 - \alpha) \hat \varphi_x(i) + \sum_{l=i+1}^{j-1} \hat \varphi_x(l) + \beta \hat \varphi_x(j)} = \\
    &\norm{(1 - \alpha) \hat \varphi_x(i) + (1 - \beta) \sum_{l=i+1}^{j-1} \hat \varphi_x(l) + \beta \sum_{l=i+1}^{j} \hat \varphi_x(l)} =\\
    & \norm{(1 - \beta) \left( (1 - \alpha) \hat \varphi_x(i) + \sum_{l=i+1}^{j-1} \hat \varphi_x(l) \right) + \\
      &\beta \left( (1 - \alpha) \hat \varphi_x(i) + \sum_{l=i+1}^{j} \hat \varphi_x(l) \right)} \leq ... 
  \end{aligned}
\end{equation}
\begin{equation}
  \begin{aligned}
    ... \leq &(1-\beta) \norm{(1 - \alpha) \hat \varphi_x(i) + \sum_{l=i+1}^{j-1} \hat \varphi_x(l)} + \\
    &\beta \norm{(1 - \alpha) \hat \varphi_x(i) + \sum_{l=i+1}^{j} \hat \varphi_x(l)} \overset{\eqref{app_eq:beta0},\eqref{app_eq:beta1}} \leq\\
    & \frac{1}{h} (1-\beta) \kappa(\gamma_j^0 - \gamma_i^{\alpha})  + \beta \kappa(\gamma_j^1 - \gamma_i^{\alpha})\overset{\text{concavity}}\leq\\
    & \frac{1}{h} \kappa((1-\beta) (\gamma_j^0 - \gamma_i^{\alpha}) + \beta (\gamma_j^1 - \gamma_i^{\alpha}))=
     \frac{1}{h} \kappa(\gamma_j^{\beta} - \gamma_i^{\alpha})
  \end{aligned}
\end{equation}
%Hence, it is enough to consider \eqref{app_eq:infinite_jump_constraints} for 
Noticing that \eqref{app_eq:kappa_constraints} is precisely \eqref{app_eq:infinite_jump_constraints} for $\alpha, \beta
\in \{ 0, 1 \}$ (as $\kappa(a)=0 \Leftrightarrow a = 0$) completes the proof.
%Hence it is enough to consider the constraints for $\alpha, \beta \in \{ 0, 1\}$.
\end{proof}

\begin{prop}% \footnote{\label{note1} The proposition slightly differs from the one in the paper, and will be adapted accordingly in the final version of the manuscript.}
For convex one-homogeneous $\eta$ the discretization with piecewise constant $\varphi_t$ and $\varphi_x$ leads 
to the traditional discretization as proposed in \cite{PCBC-ICCV09}, except with min-pooled instead of sampled unaries. 
\label{prop:equiv_standard}
\end{prop}
\begin{proof}
  The constraints in \cite[Eq.~18]{PCBC-ICCV09} have the form
  \begin{align}
    &\hat \varphi_t(i) \geq \eta^*(\hat \varphi_x(i)) - \rho(\gamma_i), \label{app_eq:constraints_PCBC1} \\
    &\bigl \| \sum_{l=i}^j \hat \varphi_x(l) \bigr \| \leq \kappa(\gamma_{j+1} - \gamma_i), \label{app_eq:constraints_PCBC2}
  \end{align}
with $\rho(u) = \lambda (u-f)^2$, $\eta(g) = \norm{g}^2$ and $\kappa(a) = \nu \iver{a > 0}$. 
The constraints \eqref{app_eq:constraints_PCBC2} are equivalent to \eqref{app_eq:kappa_constraints} up to a rescaling of $\hat \varphi_x$ with $h$. 
For the constraints \eqref{app_eq:constraints_PCBC1} (cf. \cite[Eq.~18]{PCBC-ICCV09}), the unaries are sampled at the labels $\gamma_i$. The discretization with piecewise constant duals leads to a similar form, except for a min-pooling on dual intervals, $\forall 1 \leq i \leq k$:
\begin{equation}
  \begin{aligned}
    \hat \varphi_t(i) &\geq 
  \eta^*(\hat \varphi_x(i)) - \inf_{t \in [\gamma_i, \gamma_i^*]} \rho(t), \\ 
    \hat \varphi_t(i + 1) &\geq 
  \eta^*(\hat \varphi_x(i)) - \inf_{t \in [\gamma_i^*, \gamma_{i+1}]} \rho(t).
   \end{aligned}
  \label{app_eq:discretized_separable_constraints_00}
\end{equation}
The similarity between \eqref{app_eq:discretized_separable_constraints_00} and \eqref{app_eq:constraints_PCBC1} becomes more evident by assuming convex one-homogeneous $\eta$. Then \eqref{app_eq:discretized_separable_constraints_00} reduces to the following:
\begin{equation}
  \begin{aligned}
    -\hat \varphi_t(1) \leq &\inf_{t \in [\gamma_1, \gamma_1^*]} \rho(t), \\
    -\hat \varphi_t(i) \leq &\inf_{t \in \Gamma_i^*} \rho(t), ~ \forall i \in \{ 2, \hdots, \ell - 1 \}, \\
    -\hat \varphi_t(\ell) \leq &\inf_{t \in [\gamma_{\ell-1}^*, \gamma_\ell]} \rho(t),
  \end{aligned}
\end{equation}
as well as 
\begin{equation}
  \hat \varphi_x(i) \in \dom ( \eta^* ), \forall i \in \{ 1, \hdots, k \}.
\end{equation}
 
\end{proof}

\begin{prop}

\label{prop:constraints}

The constraints 
\begin{equation}
  \begin{aligned}
    \inf_{t \in \Gamma_i} ~&\hat \varphi_t(i) \frac{\gamma_{i+1} -
      t}{h} + \hat \varphi_t(i+1)
    \frac{t - \gamma_{i}}{h} \\
    &+ \rho(t) \geq \eta^*(\hat \varphi_x(i)).
  \end{aligned}
  \label{app_eq:constraints_sublabel}
\end{equation}
can be equivalently reformulated by 
introducing additional variables $a \in \bbR^k$, $b \in \bbR^k$, where $\forall i \in \{ 1,  \hdots, k \}$:
\begin{equation}
  \begin{aligned}
 &r(i) = (\hat \varphi_t(i) - \hat \varphi_t(i+1)) / h,\\
 &a(i) + b(i) - (\hat \varphi_t(i) \gamma_{i+1} - \hat \varphi_t(x, i+1)
   \gamma_{i}) / h = 0,\\ 
   &r(i) \geq \rho_i^* \left( a(i) \right), \hat \varphi_x(i) \geq \eta^* \left( b(i) \right),
  \end{aligned}
  \label{app_eq:piecw_lin_constraints}
\end{equation}
with $\rho_i(x, t) = \rho(x, t) + \delta\{ t \in \Gamma_i \}$.
\end{prop}
\begin{proof}
  Rewriting the infimum in \eqref{app_eq:constraints_sublabel} as minus the convex conjugate of $\rho_i$, and multiplying the
  inequality with $-1$ the constraints become:
\begin{equation}
  \begin{aligned}
    &\rho_i^*(r(i)) + \eta^*(\hat \varphi_x(i)) - \frac{\hat \varphi_t(i) \gamma_{i+1} - \hat \varphi_t(i+1) \gamma_i}{h} \leq 0, \\
    &r(i) = (\hat \varphi(i) - \hat \varphi(i+1)) / h.
  \end{aligned}
  \label{app_eq:constraints_sublabell_2}
\end{equation}
To show that \eqref{app_eq:constraints_sublabell_2} and \eqref{app_eq:piecw_lin_constraints} are equivalent, we prove 
that they imply each other. 
%\underline{$\eqref{app_eq:constraints_sublabell_2} \Rightarrow \eqref{app_eq:piecw_lin_constraints}$:}\\
Assume \eqref{app_eq:constraints_sublabell_2} holds. Then without loss of generality
set $a(i) = \rho_i^*(r(i)) + \xi_1$, $b(i) = \eta_i^*(\varphi_x(i)) + \xi_2$ for some $\xi_1,\xi_2 \geq 0$. Clearly, this choice fulfills $\eqref{app_eq:constraints_sublabell_2}$. Since for $\xi_1 = \xi_2 = 0$ we have by assumption that
\begin{equation}
 a(i) + b(i) - (\hat \varphi_t(i) \gamma_{i+1} - \hat \varphi_t(x, i+1)
   \gamma_{i}) / h \leq 0,\\ 
\end{equation}
there exists some $\xi_1, \xi_2 \geq 0$ such that  \eqref{app_eq:piecw_lin_constraints} holds.

%\underline{$\eqref{app_eq:piecw_lin_constraints} \Rightarrow \eqref{app_eq:constraints_sublabell_2}$:}\\
Now conversely assume \eqref{app_eq:piecw_lin_constraints} holds. Since $a(i) \geq \rho_i^* \left( r(i) \right)$, $b(i) \geq \eta^* \left( \hat  \varphi_x(i) \right)$, 
and 
\begin{equation}
 a(i) + b(i) - (\hat \varphi_t(i) \gamma_{i+1} - \hat \varphi_t(x, i+1)
   \gamma_{i}) / h = 0,
\end{equation} this directly implies 
\begin{equation}
\rho_i^*(r(i)) + \eta^*(\hat \varphi_x(i)) - \frac{\hat \varphi_t(i) \gamma_{i+1} - \hat \varphi_t(i+1) \gamma_i}{h} \leq 0,
\end{equation}
since the left-hand side becomes smaller by plugging in the lower bound.
\end{proof}

\begin{prop}
\label{prop:cvpr_equiv}
The discretization with piecewise linear $\varphi_t$ and piecewise constant $\varphi_x$ together  with the 
choice $\eta(g) = \norm{g}$ and $\kappa(a) = a$ is equivalent to the relaxation \cite{moellenhoff-laude-cvpr-2016}.
%
%The discretization with piecewise linear $\varphi_t$ and piecewise constant $\varphi_x$ along with the 
%choice $\eta(g) = \norm{g}$ and $d(t, t') = |t - t'|$ leads exactly to the relaxation proposed in \cite{moellenhoff-laude-cvpr-2016}.
\label{prop:equiv_sublabel}
\end{prop}
\begin{proof}
  Since $\eta(g) = \norm{g}$, the constraints \eqref{app_eq:constraints_sublabel} become
  \begin{equation}
    \begin{aligned}
      &\inf_{t \in \Gamma_i} ~\hat \varphi_t(i) \frac{\gamma_{i+1} -
        t}{h} + \hat \varphi_t(i+1)
      \frac{t - \gamma_{i}}{h} 
      + \rho(t) \geq 0.\\
      &\varphi_x \in \dom(\eta^*).
    \end{aligned}
    \label{app_eq:constraints_sublabel_2}
  \end{equation}
  This decouples the constraints into data term and regularizer. The data term constraints can be written using the convex conjugate of $\rho_i = \rho + \delta\{ \cdot \in \Gamma_i \}$ as the following:
  \begin{equation}
    \begin{aligned}
      \frac{\hat \varphi_t(i) \gamma_{i+1} - \hat \varphi_t(i+1) \gamma_i}{h} - \rho_i^* \left( \frac{\hat \varphi_t(i) - \hat \varphi_t(i+1)}{h}  \right) \geq 0.
    \end{aligned}
    \label{app_eq:constraints_sublabel_3}
  \end{equation}
  Let $\vl_i = \hat \varphi_t(i) - \hat \varphi_t(i+1)$ and $q = \hat \varphi_t(1)$. Then we can write \eqref{app_eq:constraints_sublabel_3} as
  a telescope sum over the $\vl_i$
  \begin{equation}
    \begin{aligned}
      &q - \sum_{j=1}^{i-1} \vl_j + \frac{\gamma_i}{h} \vl_i  - \rho_i^* \left( \frac{\vl_i}{h}  \right) \geq 0, \\
    \end{aligned}
    \label{app_eq:constraints_sublabel_4}
  \end{equation}
  which is the same as the constraints in \cite[Eq.~9,Eq.~22]{moellenhoff-laude-cvpr-2016}. The cost function is given as
  \begin{equation}
    \begin{aligned}
      -\hat \varphi_t(1) - \sum_{i=1}^k \hat v(i) \left[ \hat \varphi_t(i+1) - \hat \varphi_t(i) \right] = \iprod{\hat v}{\vl} - q,
    \end{aligned}
    \label{app_eq:cost_fun}
  \end{equation}
  which is exactly the first part of \cite[Eq.~21]{moellenhoff-laude-cvpr-2016}. %, see also \eqref{app_eq:le_dataterm}.
  Finally, for the  regularizer we get
  \begin{equation}
    \begin{aligned}
      \bigl \| \sum_{l=i}^j \hat \varphi_x(l) \bigr \| \leq \frac{|\gamma_{j+1} - \gamma_i|}{h}, 
      ~ \norm{\hat \varphi_x(i)} \leq 1,
    \end{aligned}
  \end{equation}
  which clearly reduces to the same set as in \cite[Proposition~5]{moellenhoff-laude-cvpr-2016}, by applying that proposition (and with the rescaling/substitution $p = h \cdot \varphi_x$).
 
\end{proof}

\begin{prop}
  The data term from \cite{moellenhoff-laude-cvpr-2016} (which is in turn a special case of the discretization with piecewise linear $\varphi_t$) can be (pointwise) brought into the primal form
  \begin{equation}
    \mathcal{D}(\widehat v) = \inf_{\substack{x_i \geq 0,\sum_i x_i=1\\\widehat v = y / h + I^\top x}} ~ \sum_{i=1}^k x_i \rho_i^{**} \left(\frac{y_i}{x_i} \right), 
    \label{app_eq:dataterm_zach}
  \end{equation}
  where $I \in \bbR^{k \times k}$ is a discretized integration operator. 
  \label{prop:zach_equiv}
\end{prop}
\begin{proof}
The equivalence of the sublabel accurate data term proposed in \cite{moellenhoff-laude-cvpr-2016} to 
the discretization with piecewise linear $\varphi_t$ is established in Proposition~\ref{prop:cvpr_equiv} (cf. \eqref{app_eq:constraints_sublabel_4} and \eqref{app_eq:cost_fun}). It 
is given pointwise as
\begin{equation}
  \begin{aligned}
    \mathcal{D}(\widehat v) &= \max_{\vl, q} ~ \iprod{\vl}{\widehat v} - q -\\
    &\sum_{i=1}^k \delta \left \{ \left( \frac{\vl_i}{h}, \left[ q \mathbf{1}_k - I \vl \right]_i \right) \in \epi(\rho_i^*) \right \},
  \end{aligned}
  \label{app_eq:le_dataterm}
\end{equation}
where $\widehat v \in \bbR^k, \vl \in \bbR^k, q \in \bbR$, and $k$ is the number 
of pieces and $\mathbf{1}_k \in \bbR^k$ is the vector consisting only of ones. Furthermore, $\rho_i(t) = \rho(t) + \delta \{ t \in \Gamma_i \}, \dom(\rho_i) = \Gamma_i = [\gamma_i, \gamma_{i+1}]$.
The integration operator $I \in \bbR^{k \times k}$ is defined as
\begin{equation}
  I = \begin{bmatrix}
    -\frac{\gamma_1}{h} & & & &\\
    1 & -\frac{\gamma_2}{h} & & & \\
%    1 & 1 & \gamma_3 & & \\
    & & \ddots & &\\
    1 & \hdots & 1 & -\frac{\gamma_k}{h}
    \end{bmatrix}.
\end{equation}
Using convex duality, and the substitution $h \tilde v = \vl$ we can rewrite \eqref{app_eq:le_dataterm} as
\begin{equation} 
  \begin{aligned}
    \min_{x} ~ \max_{\tilde v, q, z} ~ &\iprod{\tilde v}{h \cdot \widehat v} - q - \iprod{x}{z - (q \mathbf{1}_k - h I \tilde v)} - \\
    &\sum_{i=1}^k \delta \left \{ \left( \tilde v_i, z_i \right) \in \epi(\rho_i^*) \right \},
  \end{aligned}
  \label{app_eq:le_dataterm2}
\end{equation}
The convex conjugate of $F_i(z, v) = \delta \{ (v, -z) \in \epi(\rho_i^*) \}$ is the lower-semicontinuous envelope 
of the perspective \cite[Section~15]{Rockafellar:ConvexAnalysis}, and since $\rho_i : \Gamma_i \to \bbR$ has bounded domain, is given as
the following (cf. also \cite[Appendix~3]{Zach-Kohli-eccv12})
\begin{equation}
  F_i^*(x, y) = 
  \begin{cases}
    x \rho_i^{**} (y/x), &\text{ if } x > 0,\\
    0, &\text{ if } x = 0 \wedge y = 0, \\
    \infty, &\text{ if } x < 0 \vee (x = 0 \wedge y \neq 0).
  \end{cases}
\end{equation}
Thus with the convention that $0 / 0 = 0$ equation \eqref{app_eq:le_dataterm2} can be rewritten as convex conjugates:
\begin{equation} 
  \begin{aligned}
    &\min_{x} ~ \left( \max_q q (\mathbf{1}_k^\top x) - q \right) + \\
    & \left( \max_{\tilde v, z} ~ \iprod{\tilde v}{h \cdot (\widehat v - I^\top x)} + \iprod{-z}{x} - \sum_{i=1}^k F_i(-z_i, \tilde v_i) \right) = \\
    &\min_{x} ~ \delta \left \{ \sum_i x_i = 1 \right \} + \sum_i F_i^* \left( x_i, \left[h (\widehat v - I^\top x) \right]_i \right).
  \end{aligned}
  \label{app_eq:le_dataterm3}
\end{equation}
Hence we have that
\begin{equation}
    \mathcal{D}(\widehat v) = \min_{\substack{x,y\\y = h (\widehat v - I^\top x) \\ \sum_i x_i=1 , x_i \geq 0\\ y_i/x_i \in \dom(\rho_i^{**})}} ~ \sum_i x_i \rho_i^{**} \left(\frac{y_i}{x_i} \right),
\end{equation}
which can be rewritten in the form \eqref{app_eq:le_dataterm}.
\end{proof}

\begin{prop}
  Let $\gamma  = \kappa(\gamma_2 - \gamma_1)$ and $\ell = 2$. The approximation with piecewise linear $\varphi_t$ and piecewise constant $\varphi_x$ of the continuous optimization problem 
\begin{equation}
  \inf_{v \in \mathcal{C}} ~ \sup_{\varphi \in \mathcal{K}} ~ \int_{\Omega \times \bbR} \iprod{\varphi}{Dv}.
  \label{app_eq:lifted_relaxed_cont_mshah}
\end{equation}
is equivalent to
  \begin{equation}
    \inf_{u : \Omega \to \Gamma} \int_{\Omega} \rho^{**}(x, u(x)) + (\eta^{**}
    ~ \square ~ \gamma \norm{\cdot})
  (\nabla u(x)) ~\mathrm{d}x,
  \label{app_eq:unlifted_prob}
  \end{equation}
  where $(\eta ~ \square ~ \gamma \norm{\cdot})(x) = \inf_{y} ~ \eta(x - y) + \gamma \norm{y}$ denotes the infimal convolution (cf. \cite[Section~5]{Rockafellar:ConvexAnalysis}).
  \label{prop:infconv}
\end{prop}
\begin{proof}
  Plugging 
  in the representations for piecewise linear $\varphi_t$ and piecewise constant $\varphi_x$ we have the coefficient functions
  $\hat v : \Omega \to [0,1]$, $\hat \varphi_t : \Omega \times \{1, 2\} \to \bbR$, $\hat \varphi_x : \Omega \to \bbR^n$ and the following optimization problem:
  \begin{equation}
    \begin{aligned}
    \inf_{\hat v} \sup_{\hat \varphi_x, \hat \varphi_t} ~ \int_{\Omega} & -\hat\varphi_t(x,1) - \hat v(x) \left[ \hat \varphi_t(x,2) - \hat \varphi_t(x,1) \right] \\
    &- h \cdot \hat v(x) \cdot \Div_x \hat \varphi_x(x) \, \mathrm{d}x \\
    &\text{subject to}\\
    &\hspace{-2cm}\inf_{t \in \Gamma} \hat \varphi_t(x,1) \frac{\gamma_2 - t}{h} + \hat
    \varphi_t(x,2) \frac{t - \gamma_1}{h} + \rho(x,t) \geq \eta^*(x, \hat \varphi_x(x))\\
    &\hspace{-2cm}\norm{\hat \varphi_x(x)} \leq \kappa(\gamma_2 - \gamma_1) =: \gamma.
    \end{aligned}   
    \label{app_eq:opti_prob}
  \end{equation}
  Using the convex conjugate of $\rho : \Omega \times \Gamma \to \bbR$ (in its second argument), we rewrite the first constraint as
    \begin{equation}
      \begin{aligned}
      &\frac{\hat \varphi_t(x,1) \gamma_2 - \hat \varphi_t(x, 2) \gamma_1}{h} \geq \\
      &\qquad \rho^* \left( x, \frac{ \hat \varphi_t(x,1) -  \hat \varphi_t(x,2)}{h} \right) + \eta^*(x, \hat \varphi_x(x)).
      \end{aligned}
    \end{equation}
    Using the substitution $\tilde \varphi(x) = \frac{\hat \varphi_t(x,1) - \hat \varphi_t(x,2)}{h}$ we can reformulate the constraints as
    \begin{equation}
      \hat \varphi_t(x,1) \geq \rho^*(x, \tilde \varphi(x)) + \eta^*(x, \hat \varphi_x(x)) - \gamma_1 \tilde \varphi(x),
      \label{app_eq:const}
    \end{equation}
    and the cost function as
    \begin{equation}
      \sup_{\tilde \varphi, \hat \varphi_t, \hat \varphi_x}\int_{\Omega} -\hat \varphi_t(x,1) + h \hat v(x) \tilde \varphi(x) - h \hat v (x) \Div_x \hat \varphi_x(x) \mathrm{d}x.
    \end{equation}
    The pointwise supremum over $-\hat \varphi_t(x,1)$ is attained where the constraint \eqref{app_eq:const} is sharp, which means we can pull it into the cost function to arrive at
    \begin{equation}
      \begin{aligned}
        &\sup_{\tilde \varphi, \hat \varphi_x}\int_{\Omega} -\rho^*(x, \tilde \varphi(x)) - \eta^*(x, \hat \varphi_x(x)) - \delta\{ \norm{\hat \varphi_x(x) \leq \gamma } \}+ \\
        &\qquad \gamma_1 \tilde \varphi(x) + h \hat v(x) \tilde \varphi(x) - h \hat v(x) \Div_x \hat \varphi_x(x) \mathrm{d}x,
      \end{aligned}
    \end{equation}
    where we wrote the second constraint in \eqref{app_eq:opti_prob} as an indicator function.
    As the supremum decouples in $\tilde \varphi$ and $\hat
    \varphi_x$, we can rewrite it using convex (bi-)conjugates, by
    interchanging integration and supremum (cf. \cite[Theorem~14.60]{VariAna}):
    \begin{equation}
      \begin{aligned}
      \sup_{\tilde \varphi} \int_{\Omega} &\gamma_1 \tilde \varphi(x) + h \hat v(x) \tilde \varphi(x) - \rho^*(x, \tilde \varphi(x)) \mathrm{d}x = \\
      &\int_{\Omega} \sup_{\tilde \varphi} ~ \gamma_1 \tilde \varphi + h \hat v(x) \tilde \varphi - \rho^*(x, \tilde \varphi) \mathrm{d}x = \\
      &\int_{\Omega} \rho^{**}(x, \gamma_1 + h \hat v(x)) ~\mathrm{d}x.
      \end{aligned}
      \label{app_eq:cvx_dt}
    \end{equation}
    For the part in $\hat \varphi_x$ we assume that $\hat v$ is
    sufficiently smooth and apply partial integration ($\hat \varphi_x$
    vanishes on the boundary), and then perform a similar calculation to
    the previous one:
    \begin{equation}
      \begin{aligned}
      \sup_{\hat \varphi_x} &\int_{\Omega} -(\eta^* + \delta \{
      \norm{\cdot} \leq \gamma \})(x, \hat \varphi_x(x)) - \\
      &\qquad h \hat v(x) \Div_x \hat \varphi_x(x) \mathrm{d}x = \\
      \sup_{\hat \varphi_x} &\int_{\Omega} -(\eta^* + \delta \{
      \norm{\cdot} \leq \gamma \})(x, \hat \varphi_x(x)) + \\
      &\qquad h \iprod{\nabla_x \hat v(x)}{\hat \varphi_x(x)} \mathrm{d}x = \\
       &\int_{\Omega} \sup_{\hat \varphi_x} -(\eta^* + \delta \{
       \norm{\cdot} \leq \gamma \})(x, \hat \varphi_x) +\\
       &\qquad h \iprod{\nabla_x \hat v(x)}{\hat \varphi_x} \mathrm{d}x = \\
      &\int_{\Omega} (\eta^* + \delta\{ \norm{\cdot} \leq \gamma \})^*(x, h \nabla_x \hat v(x)) \mathrm{d}x =\\
      &\int_{\Omega} (\eta^{**} ~\square~ \gamma \norm{\cdot})(x, h \nabla_x \hat v(x)) \mathrm{d}x = \\
      &\int_{\Omega} (\eta ~\square~ \gamma \norm{\cdot})(x, h \nabla_x \hat v(x)) \mathrm{d}x .
      \end{aligned}
      \label{app_eq:cvx_reg}
    \end{equation}
    Here we used also the result that $(f^* + g)^* = f^{**} ~\square~ g^{*}$
    \cite[Theorem~11.23]{VariAna}.
    Combining \eqref{app_eq:cvx_dt} and \eqref{app_eq:cvx_reg} and using the substitution $u = \gamma_1 + h \hat v$, we finally arrive at:
    \begin{equation}
      \int_{\Omega} \rho^{**}(x, u(x)) + (\eta^{**}~\square~\gamma \norm{\cdot})(x, \nabla u(x)) \, \mathrm{d}x,
    \end{equation}
    which is the same as \eqref{app_eq:unlifted_prob}.
\end{proof}

{\small
\bibliographystyle{ieee}
\bibliography{references}

\begin{thebibliography}{10}\itemsep=-1pt

\bibitem{ABDM}
G.~Alberti, G.~Bouchitt{\'e}, and G.~Dal~Maso.
\newblock The calibration method for the {M}umford-{S}hah functional and
  free-discontinuity problems.
\newblock {\em Calc. Var. Partial Differential Equations}, 16(3):299--333,
  2003.

\bibitem{BV}
L.~Ambrosio, N.~Fusco, and D.~Pallara.
\newblock {\em Functions of Bounded Variation and Free Discontinuity Problems}.
\newblock Oxford University Press, USA, 2000.

\bibitem{bach2015submodular}
F.~Bach.
\newblock Submodular functions: from discrete to continous domains.
\newblock {\em arXiv:1511.00394}, 2015.

\bibitem{Blake-Zisserman-87}
A.~Blake and A.~Zisserman.
\newblock {\em Visual Reconstruction}.
\newblock MIT Press, 1987.

\bibitem{Blekherman}
G.~Blekherman, P.~A. Parrilo, and R.~R. Thomas.
\newblock {\em Semidefinite Optimization and Convex Algebraic Geometry}.
\newblock SIAM, 2012.

\bibitem{bouchitte1998}
G.~Bouchitt{\'e}.
\newblock Recent convexity arguments in the calculus of variations.
\newblock {\em Lecture notes from the 3rd Int. Summer School on the Calculus of
  Variations, Pisa}, 1998.

\bibitem{bouchitte2015duality}
G.~Bouchitt{\'e} and I.~Fragal{\`a}.
\newblock Duality for non-convex variational problems.
\newblock {\em Comptes Rendus Mathematique}, 353(4):375--379, 2015.

\bibitem{carioni2016discrete}
M.~Carioni.
\newblock A discrete coarea-type formula for the {M}umford-{S}hah functional in
  dimension one.
\newblock {\em arXiv preprint arXiv:1610.01846}, 2016.

\bibitem{ChJCA}
A.~Chambolle.
\newblock Convex representation for lower semicontinuous envelopes of
  functionals in {$L^1$}.
\newblock {\em J. Convex Anal.}, 8(1):149--170, 2001.

\bibitem{Chambolle-et-al-siims12}
A.~Chambolle, D.~Cremers, and T.~Pock.
\newblock A convex approach to minimal partitions.
\newblock {\em SIAM J. Imaging Sciences}, 5(4):1113--1158, 2012.

\bibitem{fix2014duality}
A.~Fix and S.~Agarwal.
\newblock Duality and the continuous graphical model.
\newblock In {\em Proceedings of the European Conference on Computer Vision,
  {ECCV}}, 2014.

\bibitem{goldluecke2013tight}
B.~Goldluecke, E.~Strekalovskiy, and D.~Cremers.
\newblock Tight convex relaxations for vector-valued labeling.
\newblock {\em SIAM J. Imaging Sciences}, 6(3):1626--1664, 2013.

\bibitem{Hirani2003}
A.~N. Hirani.
\newblock {\em Discrete exterior calculus}.
\newblock PhD thesis, California Institute of Technology, 2003.

\bibitem{Ishikawa}
H.~Ishikawa.
\newblock Exact optimization for {M}arkov random fields with convex priors.
\newblock {\em IEEE Trans. Pattern Analysis and Machine Intelligence},
  25(10):1333--1336, 2003.

\bibitem{laude16eccv}
E.~Laude, T.~M\"ollenhoff, M.~Moeller, J.~Lellmann, and D.~Cremers.
\newblock Sublabel-accurate convex relaxation of vectorial multilabel energies.
\newblock In {\em Proceedings of the European Conference on Computer Vision,
  {ECCV}}, 2016.

\bibitem{Lellmann-Schnoerr-siims11}
J.~Lellmann and C.~Schn{\"o}rr.
\newblock Continuous multiclass labeling approaches and algorithms.
\newblock {\em SIAM J. Imaging Sciences}, 4(4):1049--1096, 2011.

\bibitem{lellmann-et-al-iccv2013}
J.~Lellmann, E.~Strekalovskiy, S.~Koetter, and D.~Cremers.
\newblock Total variation regularization for functions with values in a
  manifold.
\newblock In {\em Proceedings of the IEEE International Conference on Computer
  Vision, {ICCV}}, 2013.

\bibitem{moellenhoff-laude-cvpr-2016}
T.~M\"{o}llenhoff, E.~Laude, M.~Moeller, J.~Lellmann, and D.~Cremers.
\newblock Sublabel-accurate relaxation of nonconvex energies.
\newblock In {\em Proceedings of the IEEE Conference on Computer Vision and
  Pattern Recognition, {CVPR}}, 2016.

\bibitem{MumShah}
D.~Mumford and J.~Shah.
\newblock Optimal approximations by piecewise smooth functions and associated
  variational problems.
\newblock {\em Comm. Pure Appl. Math.}, 42(5):577--685, 1989.

\bibitem{PCBC-ICCV09}
T.~Pock, D.~Cremers, H.~Bischof, and A.~Chambolle.
\newblock An algorithm for minimizing the piecewise smooth {M}umford-{S}hah
  functional.
\newblock In {\em Proceedings of the IEEE International Conference on Computer
  Vision, {ICCV}}, 2009.

\bibitem{PCBC-SIIMS}
T.~Pock, D.~Cremers, H.~Bischof, and A.~Chambolle.
\newblock Global solutions of variational models with convex regularization.
\newblock {\em SIAM J. Imaging Sci.}, 3(4):1122--1145, 2010.

\bibitem{PZB07}
T.~Pock, C.~Zach, and H.~Bischof.
\newblock {M}umford-{S}hah meets stereo: Integration of weak depth hypotheses.
\newblock In {\em Proceedings of the IEEE Conference on Computer Vision and
  Pattern Recognition, {CVPR}}, 2007.

\bibitem{Rockafellar:ConvexAnalysis}
R.~T. Rockafellar.
\newblock {\em Convex Analysis}.
\newblock Princeton University Press, 1996.

\bibitem{VariAna}
R.~T. Rockafellar, R.~J.-B. Wets, and M.~Wets.
\newblock {\em Variational analysis}.
\newblock Springer, 1998.

\bibitem{Schlesinger76}
M.~Schlesinger.
\newblock {S}intaksicheskiy analiz dvumernykh zritelnikh signalov v usloviyakh
  pomekh ({S}yntactic analysis of two-dimensional visual signals in noisy
  conditions).
\newblock {\em Kibernetika}, 4:113--130, 1976.

\bibitem{Strekalovskiy-et-al-cvpr12}
E.~Strekalovskiy, A.~Chambolle, and D.~Cremers.
\newblock A convex representation for the vectorial {M}umford-{S}hah
  functional.
\newblock In {\em Proceedings of the IEEE Conference on Computer Vision and
  Pattern Recognition, {CVPR}}, 2012.

\bibitem{strekalovskiy-et-al-siims14}
E.~Strekalovskiy, A.~Chambolle, and D.~Cremers.
\newblock Convex relaxation of vectorial problems with coupled regularization.
\newblock {\em SIAM J. Imaging Sciences}, 7(1):294--336, 2014.

\bibitem{Werner-tpami2007}
T.~Werner.
\newblock A linear programming approach to max-sum problem: A review.
\newblock {\em IEEE Trans. Pattern Analysis and Machine Intelligence},
  29(7):1165--1179, 2007.

\bibitem{windheuser2016convex}
T.~Windheuser and D.~Cremers.
\newblock A convex solution to spatially-regularized correspondence problems.
\newblock In {\em Proceedings of the European Conference on Computer Vision,
  {ECCV}}, 2016.

\bibitem{Zach-aistats13}
C.~Zach.
\newblock Dual decomposition for joint discrete-continuous optimization.
\newblock In {\em Proceedings of the International Conference on Artificial
  Intelligence and Statistics, {AISTATS}}, 2013.

\bibitem{Zach-Kohli-eccv12}
C.~Zach and P.~Kohli.
\newblock A convex discrete-continuous approach for {M}arkov random fields.
\newblock In {\em Proceedings of the European Conference on Computer Vision,
  {ECCV}}, 2014.

\end{thebibliography}
}

\end{document}